\documentclass[conference]{IEEEtran}
\newcommand{\asconv}{\overset{\text{a.s.}}{\longrightarrow}}
\newcommand{\ra}{\rightarrow}

\newcommand{\Lb}{\left[}
\newcommand{\Rb}{\right]}
\newcommand{\lb}{\left(}
\newcommand{\rb}{\right)}
\newcommand{\trace}{\textnormal{trace}}
\newcommand{\diag}{\textnormal{diag}}

\newcommand{\none}{{n_1}}
\newcommand{\ntwo}{{n_2}}
\newcommand{\nmin}{{n_{\min}}}
\newcommand{\nmax}{{n_{\max}}}

\newcommand{\Kmax}{K_{\max}}

\newcommand{\bone}{\mathbf{1}}
\newcommand{\bzero}{\mathbf{0}}

\newcommand{\bx}{\mathbf{x}}
\newcommand{\by}{\mathbf{y}}

\newcommand{\bb}{\mathbf{b}}

\newcommand{\bz}{\mathbf{z}}

\newcommand{\byt}{\widetilde{\mathbf{y}}}

\newcommand{\bnu}{\boldsymbol{\nu}}

\newcommand{\bP}{\mathbf{P}}
\newcommand{\bB}{\mathbf{B}}

\newcommand{\bX}{\mathbf{X}}
\newcommand{\bY}{\mathbf{Y}}

\newcommand{\bI}{\mathbf{I}}

\newcommand{\bO}{\mathbf{O}}
\newcommand{\bU}{\mathbf{U}}

\newcommand{\btY}{\widetilde{\mathbf{Y}}}
\newcommand{\bhY}{\widehat{\mathbf{Y}}}
\newcommand{\bLN}{\mathbf{L}_{\mathcal{N}}}

\newcommand{\bDelta}{\mathbf{\Delta}}
\newcommand{\bA}{\mathbf{A}}

\newcommand{\bM}{\mathbf{M}}
\newcommand{\bD}{\mathbf{D}}
\newcommand{\bL}{\mathbf{L}}

\newcommand{\bLambda}{\mathbf{\Lambda}}

\newcommand{\bAbar}{\mathbf{\overline{A}}}

\newcommand{\bAt}{\widetilde{\mathbf{A}}}

\newcommand{\cS}{\mathcal{S}}

\newcommand{\cG}{\mathcal{G}}
\newcommand{\cV}{\mathcal{V}}
\newcommand{\cE}{\mathcal{E}}

\newcommand{\cM}{\mathcal{M}}
\newcommand{\cF}{\mathcal{F}}

\newcommand{\bbR}{\mathbb{R}}

\ifCLASSINFOpdf
\else
\fi
\usepackage{booktabs} 
\usepackage{graphicx}
\usepackage{epstopdf}
\usepackage{algorithm}
\usepackage{algorithmicx}
\usepackage{algpseudocode}
\usepackage{color}
\usepackage{amsmath,amssymb,amsfonts,amsthm}
\usepackage{subcaption}
\usepackage{multirow}
\usepackage{verbatim}
\usepackage{url}
\usepackage{hyperref}     
\usepackage[dvipsnames]{xcolor}
\newtheorem{lemma} {Lemma}
\newtheorem{theorem} {Theorem}
\newtheorem{corollary} {Corollary}

\begin{document}
%
\title{Revisiting Spectral Graph Clustering with Generative Community Models}


\author{
	\IEEEauthorblockN{Pin-Yu Chen and Lingfei Wu}
	\IEEEauthorblockA{IBM Thomas J. Waston Research Center\\Yorktown Heights, NY 10598, USA\\
		Email: pin-yu.chen@ibm.com and wuli@us.ibm.com}
}



%


\maketitle

\begin{abstract}
The methodology of community detection can be divided into two principles: imposing a network model on a given graph, or optimizing a designed objective function. The former provides guarantees on theoretical detectability but falls short when the graph is inconsistent with the underlying model. The latter is model-free but fails to provide quality assurance for the detected communities. In this paper, we propose a novel unified framework to combine the advantages of these two principles. The presented method, SGC-GEN, not only considers the detection error caused by the corresponding model mismatch to a given graph, but also yields a theoretical guarantee on community detectability by analyzing Spectral Graph Clustering (SGC) under GENerative community models (GCMs). SGC-GEN incorporates the predictability on correct community detection with a measure of community fitness to GCMs. It resembles the formulation of supervised learning problems by enabling various community detection loss functions and model mismatch metrics. We further establish a theoretical condition for correct community detection using the normalized graph Laplacian matrix under a GCM, which provides a novel data-driven loss function for SGC-GEN. In addition, we present an effective algorithm to implement SGC-GEN, and show that the computational complexity of SGC-GEN is comparable to the baseline methods. Our experiments on 18 real-world datasets demonstrate that SGC-GEN possesses superior and robust performance compared to 6 baseline methods under 7 representative clustering metrics.
\end{abstract}



\section{Introduction}
Community detection aims to assign community labels to nodes in a graph such that the nodes in the same community share higher similarity (better connectivity) than the nodes in different communities \cite{Fortunato10}. It is essentially an unsupervised learning problem since one is only provided with the information of graph connectivity. 
Despite its unsupervised nature, recent research developments have been able to identify the informational and algorithmic limits of community detection under certain generative community models (GCMs), especially for spectral graph clustering (SGC) algorithms, such as the use of eigenvectors of the graph Laplacian matrices \cite{Luxburg07} or the modularity matrix \cite{Newman06community} for community detection. However, these analysis assuming that GCMs well match a graph may not hold in practice, which may often yield poor community detection results when there is a mismatch between the given graph and the underlying GCM. On the other hand, optimizing a designed objective function for community detection, such as normalized cut \cite{White05} or modularity \cite{Newman04mod}, imposes no model assumption but is sensitive in community detection \cite{leskovec2010empirical,yang2015defining}.

Motivated by the advantages of the theoretical and objective principles, we propose SGC-GEN, a novel unified community detection framework that possesses the following features: \\
$\bullet$ \textbf{The power of community detectability.} Under GCMs, the theoretical analysis of community detectability allows us to assess the quality of communities by converting the theoretical guarantees to a loss function that quantifies the error in community detection.\\
$\bullet$ \textbf{The constraint to model mismatch.} By imposing an error metric on the level of inconsistency between a given graph and a GCM, one can confine the detection error due to model mismatch and hence improve community detection. 

In particular, due to the extraordinary performance of SGC based on the normalized graph Laplacian matrix, a number of variants of SGC methods have been proposed to improve clustering performance in terms of scalability, robustness, and applicability. To provide a thorough analysis, in this paper we focus on the  standard formulation  of SGC based on the normalized graph Laplacian matrix introduced by the seminal works (see Sec. \ref{subsec_NGL}) \cite{Shi00,ng2002spectral,Luxburg07}. The main line of this paper is to demonstrate the effectiveness of SGC-GEN that combines standard SGC with GCMs \cite{Holland83} in an unified framework. Originated from the standard SGC formulation as presented in Sec. \ref{subsec_NGL}, SGC-GEN can  easily be generalized to many state-of-the-art SGC methods \cite{liu2013large,nie2014clustering,li2016scalable,nie2016constrained}. By revisiting the standard formulation of SGC with GCMs, we establish a novel condition on correct community detection using SGC via the normalized graph Laplacian matrix under a GCM called the stochastic block model (SBM) \cite{Holland83}. We then convert this condition to a data-driven community detection loss function and apply it to SGC-GEN to develop effective and computationally-efficient community detection methods. 
We highlight our contributions as following: \\
$\bullet$ We propose SGC-GEN, a unified community detection framework combining the principles of theoretical detectability and well-designed objective functions for improvement.  \\
$\bullet$ We establish a condition on the correctness of community detection using SGC under a SBM, which leads to a novel data-driven community loss function for SGC-GEN. Moreover, since the loss function enables community quality assessment, the proposed SGC-GEN resembles the formulation of a supervised learning problem consisting of a loss function and a regularization function. \\
$\bullet$  We present an algorithm for SGC-GEN and conduct rigorous computational analysis showing that SGC-GEN could be implemented as efficient as other baseline methods.\\
$\bullet$ We compare the performance of community detection on 18 real-life graph datasets and use 7 representative clustering metrics to rank each method. The experimental results show that joint consideration of theoretical detectability and model mismatch using SGC-GEN can substantially improve community detection when compared to 6 baseline community detection methods of similar objective functions. 


\section{General Framework}
\label{sec_framework}
\subsection{Notations}

Throughout this paper bold uppercase letters (e.g., $\bX$ or $\bX_{k}$) denote matrices
and $[\bX]_{ij}$ denotes the entry in the $i$-th row and the $j$-th
column of $\bX$, bold lowercase letters (e.g., $\bx$ or $\bx_k$) denote
column vectors, the term $\cdot^T$ denotes matrix or vector transpose, italic
letters (e.g., $x$, $x_k$ or $X$) denote scalars, and calligraphic uppercase letters
(e.g., $\mathcal{X}$ or $\mathcal{X}_i$) denote sets. The term $\cG=(\cV,\cE)$ denotes a graph characterized by a node set $\cV$ and an edge set $\cE=\{(i,j):i,j \in \cV\}$. The number of nodes and edges in $\cG$ are denoted by $n$ and $m$, respectively.
The convergence of a real rectangular matrix $\bX \in \bbR^{n_1 \times n_2}$ is with respect to the spectral norm, which is defined as $\|\bX\|_2=\max_{\bz \in \bbR^{n_2}, \bz^T \bz=1} \|\bX \bz\|_2$, where $\|\bz\|_2$ denotes the Euclidean norm of a vector $\bz$.  Based on the definition, $\|\bX\|_2$ is equivalent to the largest singular value of $\bX$.  
A matrix $\bX \in \bbR^{n_1 \times n_2}$ is said to converge to another matrix $\bM$ of the same dimension if $\|\bX-\bM\|_2$
 approaches zero. For the convenience of notation, we write $\bX \ra \bM$
if  $\|\bX-\bM\|_2 \ra 0$
as $n_1,n_2 \ra \infty$.

\subsection{Preliminaries}
Throughout this paper, we consider the problem of non-overlapping community detection in a simple connected graph that is undirected, unweighted and contains no self-loops. Given a graph $\cG=(\cV,\cE)$ and the number of communities $K$,  non-overlapping community detection aims to assign each node a community label and divide the nodes into $K$ communities such that the nodes in the same community are better connected than nodes in different communities.

\textbf{Spectral graph clustering (SGC).} SGC is a widely used technique for community detection. It transforms a graph into a vector space representation via spectral decomposition of a matrix associated with a graph. Specifically, each node in the graph is represented by a low-dimensional vector using a common subset of eigenvectors of a matrix. Based on the vector space representation, K-means clustering is applied to obtain $K$ communities . One typical example of SGC is the normalized graph Laplacian matrix \cite{Luxburg07}, where its $K$ smallest eigenvectors are used  for community detection \cite{White05}.

\textbf{Generative community model (GCM).} A GCM generates a graph that embeds community structures \cite{Goldenberg2010survey}. A GCM can be either parametric or nonparametric. For example, the stochastic block model (SBM) \cite{Holland83} is a parametric GCM that specifies a set of within-community and between-community edge connection probability parameters. The graphon model \cite{diaconis2007graph,zhang2015estimating} is a nonparametric GCM that generates a graph based on latent representations. 
Different GCMs are discussed in the survey paper \cite{Goldenberg2010survey}.

\textbf{SGC under GCMs.} 
For graphs generated by certain GCMs, recent research findings suggest that the performance of community detection using SGC can be separated into two regimes \cite{abbe2016exact}: a \textit{detectable regime} where the detected communities are consistent with the ground-truth communities, and an  \textit{undetectable regime} where the detected  communities and the ground-truth communities are inconsistent. Moreover, the critical space that separates these two  regimes can be specified. Consequently, the problem of evaluating the quality of detected communities can be converted to the problem of estimating to which regime the given graph belongs. More details are given in the related work section (Sec. \ref{sec_related}).


%

\subsection{Problem Formulation of SGC-GEN}
Consider community detection in a graph $\cG$ with an unknown number of communities. For each possible number of communities $K$, we can provide quantitative measures on community detectability and model mismatch for SGC under GCMs. Specifically, given a GCM $\cM$ of $K$ communities and a set of communities $\{\cG_k\}_{k=1}^K$ detected by a SGC method $\cF$,  the corresponding community detection loss function and model mismatch metric are as follows.

\textbf{Community detection loss function.} For any $\cF$, $\cM$ and $\{\cG_k\}_{k=1}^K$, let $f(\{\cG_k\}_{k=1}^K,\cF,\cM)$ denote a nonnegative loss function that reflects the level of incorrect community detection using $\cF$ under $\cM$. Higher loss suggests the detected communities $\{\cG_k\}_{k=1}^K$ are less reliable.

\textbf{Model mismatch metric.} Let $R(\{\cG_k\}_{k=1}^K,\cM)$ be a real-valued function quantifying the difference between the detected communities $\{\cG_k\}_{k=1}^K$ using $\cF$ and the underlying GCM $\cM$. Larger value of $R$ suggests the detected communities $\{\cG_k\}_{k=1}^K$ are less consistent with the assumption of $\cM$.

\textbf{SGC-GEN.} Inspired by the formation of supervised learning problems, community detection, albeit an unsupervised learning problem, can be formulated in a similar fashion by specifying a  community detection loss function $f$ and a model mismatch metric $R$. Given a maximum number of communities $\Kmax$, a SGC method $\cF$ and a GCM $\cM$, the proposed community detection framework, called SGC-GEN, solves the following minimization problem
\begin{align}
\label{eqn_SGC_GEN}
\min_{\{\cG_k\}_{k=1}^K \in \cS}   f(\{\cG_k\}_{k=1}^K,\cF,\cM)+ \alpha \cdot R(\{\cG_k\}_{k=1}^K,\cM), 
\end{align}
where $\cS=\{\{\cG_k\}_{k=1}^K: K=2,\ldots,\Kmax \}$ denotes the set of candidate community detection results of different number of communities obtained by $\cF$.
Using terminology from supervised learning theory,  $f$ is analog to the loss function, $R$ resembles the regularization function, and $\alpha \geq 0$ is the regularization parameter.

Many existing community detection methods  can fit into the framework of SGC-GEN in (\ref{eqn_SGC_GEN}). For example, objective-function-based algorithms specify a particular energy function $f$ for quality assessment and set $R=0$ \cite{zelnik2004self}.  Greedy  algorithms specify a model mismatch metric $R$ and set $f=0$ and $\alpha=1$. For example, the Louvain method \cite{blondel2008fast} selects $R$ to be the negative modularity, where modularity is a measure of relative difference between the detected communities and the corresponding configuration model \cite{Newman06PNAS}.

\section{Theoretical Foundation of SGC-GEN: Normalized Graph Laplacian Matrix and Stochastic Block Model}
\label{sec_SGC_GEN_thm}
In this section we study the community detectability of SGC using the normalized graph Laplacian matrix under a stochastic block model (SBM).
We establish a sufficient and necessary condition such that SGC is guaranteed to yield reliable community detection results for graphs generated by a SBM. The established condition will be used in Sec. \ref{sec_SGC_GEN_algo} to devise a novel data-driven community detection loss function for the proposed  SGC-GEN framework in (\ref{eqn_SGC_GEN}). For demonstration, we also provide a case study of the established condition under a simplified SBM. The proofs of the established theories are given in the supplementary material\footnote{Supplementary material can be downloaded from  \href{https://sites.google.com/site/pinyuchenpage/}{www.pinyuchen.com}}.

\subsection{Normalized Graph Laplacian Matrix and Stochastic Block Model (SBM)}
\label{subsec_NGL}
\textbf{SGC using normalized graph Laplacian matrix.}
Let $\bA$ denote the $n \times n$ adjacency matrix of $\cG$ and let $\bD$ be the corresponding diagonal degree matrix. The unnormalized graph Laplacian matrix is defined as $\bL=\bD-\bA$. The normalized graph Laplacian matrix is defined as $\bLN=\bD^{-\frac{1}{2}} \bL \bD^{-\frac{1}{2}}$. We denote the $k$-th smallest eigenpair of $\bLN$ by $(\lambda_k,\by_k)$, where $\by_k$ is the eigenvector associated with the eigenvalue $\lambda_k$, and $\lambda_k \leq \lambda_{k+1}$. It is also known that $\lambda_1=0$ \cite{Luxburg07}.  The standard SGC algorithm using $\bLN$ \cite{ng2002spectral} is summarized in Algorithm \ref{algo_SGC}.

Let $\btY=[\by_1 \ldots \by_K]$ be the matrix of eigenvectors $\{\by_k\}_{k=1}^K$. 
The matrix $\btY$ is the solution of the minimization problem
\begin{align}
\label{eqn_SGC_1}
\min_{\bX \in \bbR^{n \times K},~\bX^T \bX = \bI_{K}} \trace(\bX^T \bLN \bX),
\end{align}
where $\bI_K$ is the $K \times K$ identity matrix, and the constraint $\bX^T \bX = \bI_{K}$ imposes orthogonality and unit norm for the columns in $\bX$.
If $\cG$ is a connected graph, then by the definition of $\bLN$, we have $\by_1=\bD^{\frac{1}{2}} \frac{\bone_n}{\sqrt{n}}$. Let $\bY= [\by_2~\ldots~\by_K]$ be the matrix after removing the first column $\by_1$ from $\btY$. Then (\ref{eqn_SGC_1}) can be reformulated as 
\begin{align}
\label{eqn_SGC_2}
\min_{\bX \in \bbR^{n \times (K-1)},~\bX^T \bX = \bI_{K-1},~\bX^T \bD^{\frac{1}{2}} \bone_n=\bzero_{K-1} } \trace(\bX^T \bLN \bX),
\end{align}
where $\bone_n~(\bzero_n)$ is the vector of 1's (0's) and $\bY$ is the solution to (\ref{eqn_SGC_2}).
The minimization problem in (\ref{eqn_SGC_2}) is a standard formulation of SGC based on the normalized graph Laplacian matrix \cite{Shi00,ng2002spectral,Luxburg07}, which is also a fundamental  element of many state-of-the-art SGC methods \cite{liu2013large,nie2014clustering,li2016scalable,nie2016constrained}, and it  will be the foundation of the theoretical results presented in Sec. \ref{subsec_thm}.

\textbf{Stochastic block model (SBM).}
SBM \cite{Holland83} is a fundamental GCM, and it has been the root of many other GCMs such as the degree-corrected SBM \cite{Karrer11} and the random interconnection model \cite{CPY16AMOS}. SBM is a parametric GCM that assumes common edge connection probability for within-community and between-community edges.
A graph $\cG$ of $K$ communities can be generated by a SBM as follows. The SBM first divides the $n$ nodes into $K$ groups, where each group has $n_k$ nodes such that $\sum_{k=1}^n n_k = n$. For each unordered node pair $(i,j)$, $i \neq j$, an edge between $i$ and $j$  is connected with probability $P_{g_i g_j}$, where $g_i,g_j \in \{1,\ldots,K\}$ denote the community labels of $i$ and $j$. Therefore, the SBM is parameterized by the number of communities $K$ and the $K \times K$ edge connection probability matrix $\bP$, where $[\bP]_{k \ell}= P_{k \ell}$ and $\bP$ is symmetric. We denote the SBM with parameters $K$ and $\bP$ by SBM($K$,$\bP$).

\begin{algorithm}[t]
	\caption{Standard SGC using  $\bLN$ \cite{ng2002spectral} }
	\label{algo_SGC}
	\begin{algorithmic}
		\State 	\textbf{Input:} graph $\cG$, number of
		communities $K$
		\State \textbf{Output:} $K$ communities $\{\cG_k\}_{k=1}^K$	
		\State 1. Obtain $\bLN=\bD^{-\frac{1}{2}} \bL \bD^{-\frac{1}{2}}$
		\State 2. Compute $\btY=[\by_1~\by_2~\ldots~\by_K]$
		\State 3. Row normalization: $[\bhY]_{ij}=[\btY]_{ij}/\sqrt{\sum_{k=1}^K [\btY]_{ik}^2}$,  $\forall~i,j$ 		 
		\State 4. K-means clustering on the rows of $\bhY$ and output $\{\cG_k\}_{k=1}^K$			
	\end{algorithmic}
\end{algorithm}

\subsection{Theoretical Guarantees on Community Detectability}
\label{subsec_thm}
Here we analyze the performance of community detection on graphs generated by SBM($K$,$\bP$) 
using  $\bLN$. In particular, we establish a sufficient and necessary condition on correct community detection, where correct community detection means the detected communities using $\bLN$ match the oracle communities generated by SBM($K$,$\bP$), up to some permutation in community labels. 
 The condition of community detectability leads to a novel community detection loss function as will be discussed in Sec. \ref{sec_SGC_GEN_algo}.
Let $\nmin=\min_{k}{n_k}$, $\nmax=\max_{k}{n_k}$, and let $\rho_k$  denote the limit value of $\frac{n_k}{n}$  as $n_k \ra \infty$.
The following lemma serves as a cornerstone that connects the dots between  $\bLN$ and SBM($K$,$\bP$).

\begin{lemma}(matrix concentration under SBM($K$,$\bP$)) \\	
	\label{lemma_SBM_concentration}	
	Let $\bA_{ij} \in \bbR^{n_i \times n_j}$ denote the adjacency matrix of edges between communities $\cG_i$ and $\cG_j$ of a graph generated by SBM($K$,$\bP$), $i,j \in \{1,\ldots,K\}$. The following holds almost surely as $n_k \ra \infty$, $\forall~k \in \{1,\ldots,K\}$ and $\frac{\nmin}{\nmax} \ra c > 0$:
	\begin{center}
		$\frac{\bA_{ij}}{n}	\ra \sqrt{\rho_i \rho_j} P_{ij} \frac{\bone_{n_1}}{\sqrt{n_1}} \frac{\bone_{n_2}^T}{n_2}.$ 
	\end{center}
\end{lemma}
\begin{proof}	
	We use the Latala's theorem \cite{Latala05} and the Talagrand's concentration inequality \cite{Talagrand95} to prove this lemma.  The details are given in Appendix A of the supplementary material\footnotemark[1].
\end{proof}
The matrix concentration result in Lemma \ref{lemma_SBM_concentration} shows that the scaled adjacency matrix $\frac{\bA_{ij}}{n}$ converges asymptotically to a constant matrix of finite spectral norm $\sqrt{\rho_i \rho_j}P_{ij}$, which associates with the relative community size $\rho_k$ and the edge connection probability $P_{ij}$
under SBM($K$,$\bP$). The condition $c>0$ guarantees that all community sizes grow at a comparable rate.
Note that Lemma \ref{lemma_SBM_concentration} presumes each entry $P_{ij}$ in $\bP$ is a constant. In case of sparse graphs where $P_{ij}=\frac{a}{n}$ or $P_{ij}=\frac{b \log n}{n}$ for some positive constants $a,b$, similar matrix concentration result holds with high probability under mild conditions via  degree regularization techniques \cite{le2015concentration,joseph2016impact}.

Since Algorithm \ref{algo_SGC} is invariant to the permutation of node indices, for the purpose of analysis we treat the adjacency matrix $\bA$ as a matrix of $K \times K$ blocks $\{\bA_{ij}\}_{i,j=1}^K$. Using Lemma \ref{lemma_SBM_concentration}, we establish a sufficient and necessary condition on correct community detection using $\bLN$ for graphs generated by SBM($K$,$\bP$).

\begin{theorem}(community detectability using $\bLN$ under SBM($K$,$\bP$)) \\
	\label{thm_NGL_SBM}
	For any graph $\cG$ generated by SBM($K$,$\bP$),	let $\theta=\sum_{k=2}^K 1-\lambda_k$ and $\bY=[\by_2~\by_3~\ldots~\by_K]$, where $(\lambda_k,\by_k)$ is the $k$-th smallest eigenpair of $\bLN$. The following holds almost surely as $n_k \ra \infty$, $\forall~k \in \{1,\ldots,K\}$ and $\frac{\nmin}{\nmax} \ra c > 0$:
	\begin{center}
		The $K$ communities in $\cG$ can be correctly detected  \\
		using $\bY$ if and only if~$\theta>0$.
	\end{center}
\end{theorem}
\begin{proof}[Proof]
	We provide a sketch of the proof below. The complete proof is given in Appendix B of the supplementary material\footnotemark[1]. \\
	\textbf{Step 1.} Specify the optimality condition of $\bY$ using (\ref{eqn_SGC_2}).
	\\
	\textbf{Step 2.} Show the distribution of the rows in $\bY$ can be separated into two regimes, detectable or undetectable, using Lemma \ref{lemma_SBM_concentration}.
	\\
	\textbf{Step 3.} If $\bY$ is in the undetectable regime, show  the distribution of the rows in $\bY$ is inconsistent with the community structure.
	\\
	\textbf{Step 4.} If $\bY$ is in the detectable regime, show   the distribution of the rows in $\bY$ is consistent with the community structure.
	\\ 
	\textbf{Step 5.} Show $\bY$ is in the detectable regime iff $\theta>0$.
\end{proof}
Note that Theorem \ref{thm_NGL_SBM} provides a novel data-driven criterion for 
evaluating the quality of communities without the knowledge of the parameters $\bP$ in SBM($K,\bP$). In other words, for any graph generated by SBM($K,\bP$), for evaluating community detectability it suffices to compute the $K-1$ smallest nonzero eigenvalues $\{\lambda_k\}_{k=2}^K$ of $\bLN$
and inspect the condition $\theta>0$, which will be further explored in Sec. \ref{sec_SGC_GEN_algo}.
In addition, Theorem \ref{thm_NGL_SBM} also implies the feasibility of community detection using Algorithm \ref{algo_SGC}, since $\btY=[\by_1~\bY]$ and row normalization does not alter the sign of each entry in $\bhY$.

\subsection{Case Study: SBM($2,\bP$)}
To investigate the implication of the sufficient and necessary condition for correct community detection in Theorem \ref{thm_NGL_SBM},
we study  SBM($2,\bP$), the case of SBM with two communities, and justify the condition via numerical experiments. 
Under SBM($2,\bP$), we allow the size of the two communities, $n_1$ and $n_2$, to be arbitrary as long as their limit values $\rho_1, \rho_2>0$. We also simplify the notation of the edge connection matrix $\bP$ by defining $P_{11}=p_1$, $P_{22}=p_2$, and $P_{12}=P_{21}=q$. The following corollary specifies the condition of community detectability in terms of $p_1$, $p_2$ and $q$.
\begin{corollary}(community detectability using $\bLN$ under SBM($2,\bP$))
	\label{cor_SBM_two}
	For any graph $\cG$ generated by SBM($2$,$\bP$),	let  $\by_2=[\byt_1^T~\byt_2^T]^T$ denote the second smallest eigenvector of $\bLN$, where $\byt_k$, $k=1,2$, is the community-indexed block vector of $\by_2$.   The following holds almost surely as $n_1,n_2 \ra \infty$ and $\frac{\nmin}{\nmax} \ra c > 0$:
	\begin{center}
		The two communities in $\cG$ can be correctly detected  \\
		using $\bY$ if and only if~$ q < \sqrt{p_1 p_2}$.	
	\end{center}	
	Furthermore, $\byt_1 \ra \pm \beta_1 \frac{\bone_{n_1}}{\sqrt{n_1}}$ and  $\byt_2 \ra \mp \beta_2 \frac{\bone_{n_2}}{\sqrt{n_2}}$ for some $\beta_1, \beta_2 >0$ if and only if~$ q < \sqrt{p_1 p_2}$.	
\end{corollary}
\begin{proof}
	The results are induced from the condition $\theta>0$ in Theorem \ref{thm_NGL_SBM} under  SBM($2$,$\bP$).
	The proof is given in Appendix C of the supplementary material\footnotemark[1].
\end{proof}
The established detectability condition in Corollary \ref{cor_SBM_two} is universal in the sense that it does not depend on the ratio $\frac{\nmin}{\nmax}$ of the community sizes as long as its limit value $c>0$. It is worth mentioning that the condition $ q < \sqrt{p_1 p_2}$ for correct community detection is also consistent with the condition  using methods other than $\bLN$, such as the spectral modularity matrix  \cite{CPY14modularity}, the spectrum of modular matrix \cite{Peixoto13}, and the inference-based method \cite{Zhao12}.
In addition, when $ q < \sqrt{p_1 p_2}$, the results that $\byt_1 \ra \pm \beta_1  \frac{\bone_{n_1}}{\sqrt{n_1}}$ and  $\byt_2 \ra \mp \beta_2  \frac{\bone_{n_2}}{\sqrt{n_2}}$ for some $\beta_1,\beta_2>0$  imply the nodes in the same community have identical yet community-wise distinct representation, as $\byt_1$ and $\byt_2$ are nonzero constant vectors with opposite signs. This  guarantees that K-means clustering on $\by$ leads to correct community detection when  $ q < \sqrt{p_1 p_2}$. In particular,  when $n_1=n_2$ and $p_1=p_2=p$, the parameters $p$ and $q$ reflect the expected number of within-community and between-community edges, respectively. The condition in Corollary \ref{cor_SBM_two} then reduces to $q<p$, which means the two communities can be correctly detected when  there are more  within-community edges than between-community edges.

\begin{figure}[t]
	\hspace{-3mm}
	\begin{subfigure}[b]{0.53\linewidth}
		\includegraphics[width=\textwidth]{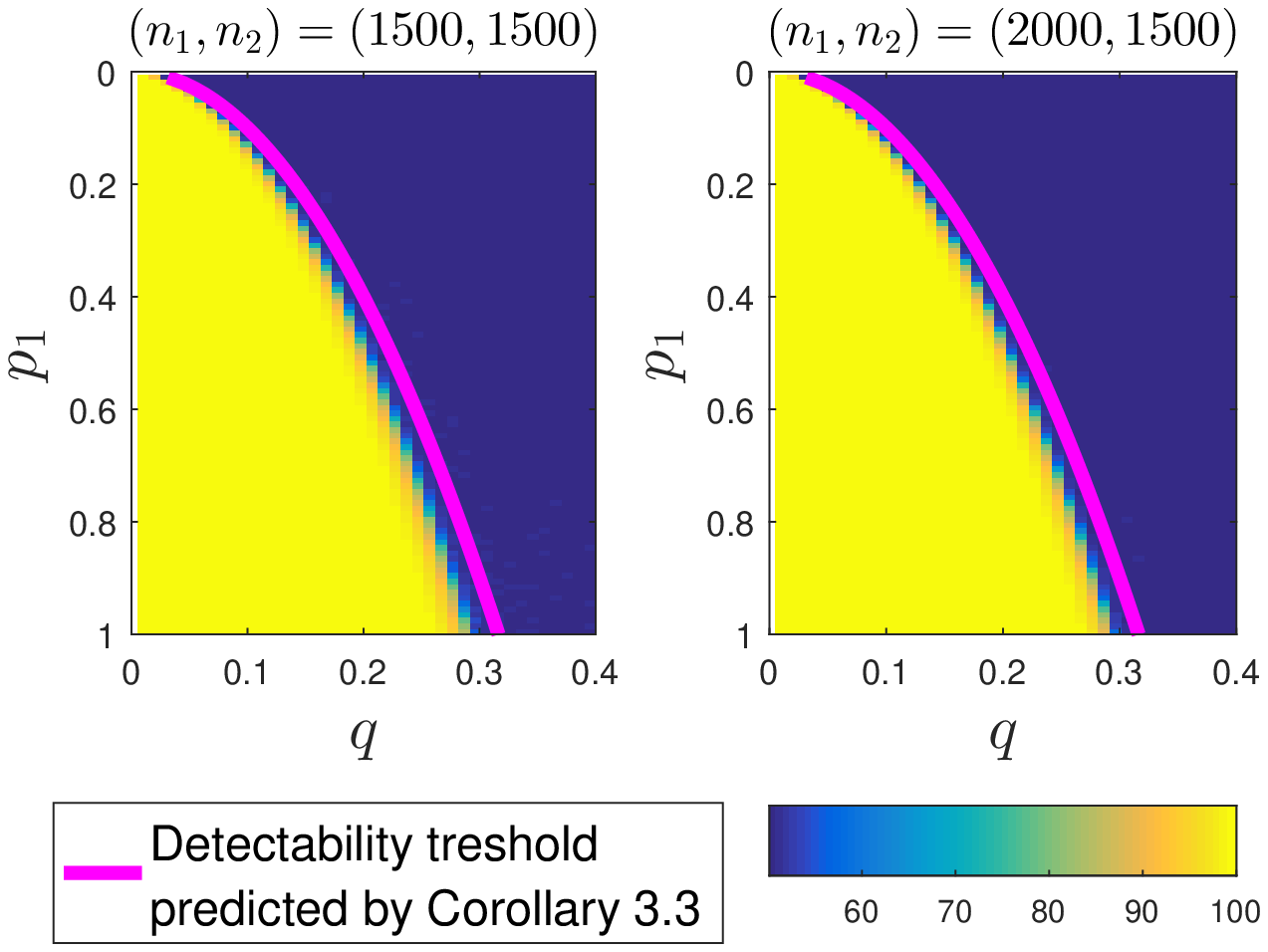}
		\vspace{-9mm}				
		\caption{Community detectability (\%).}
	\end{subfigure}%
	\begin{subfigure}[b]{0.53\linewidth}
		\includegraphics[width=\textwidth]{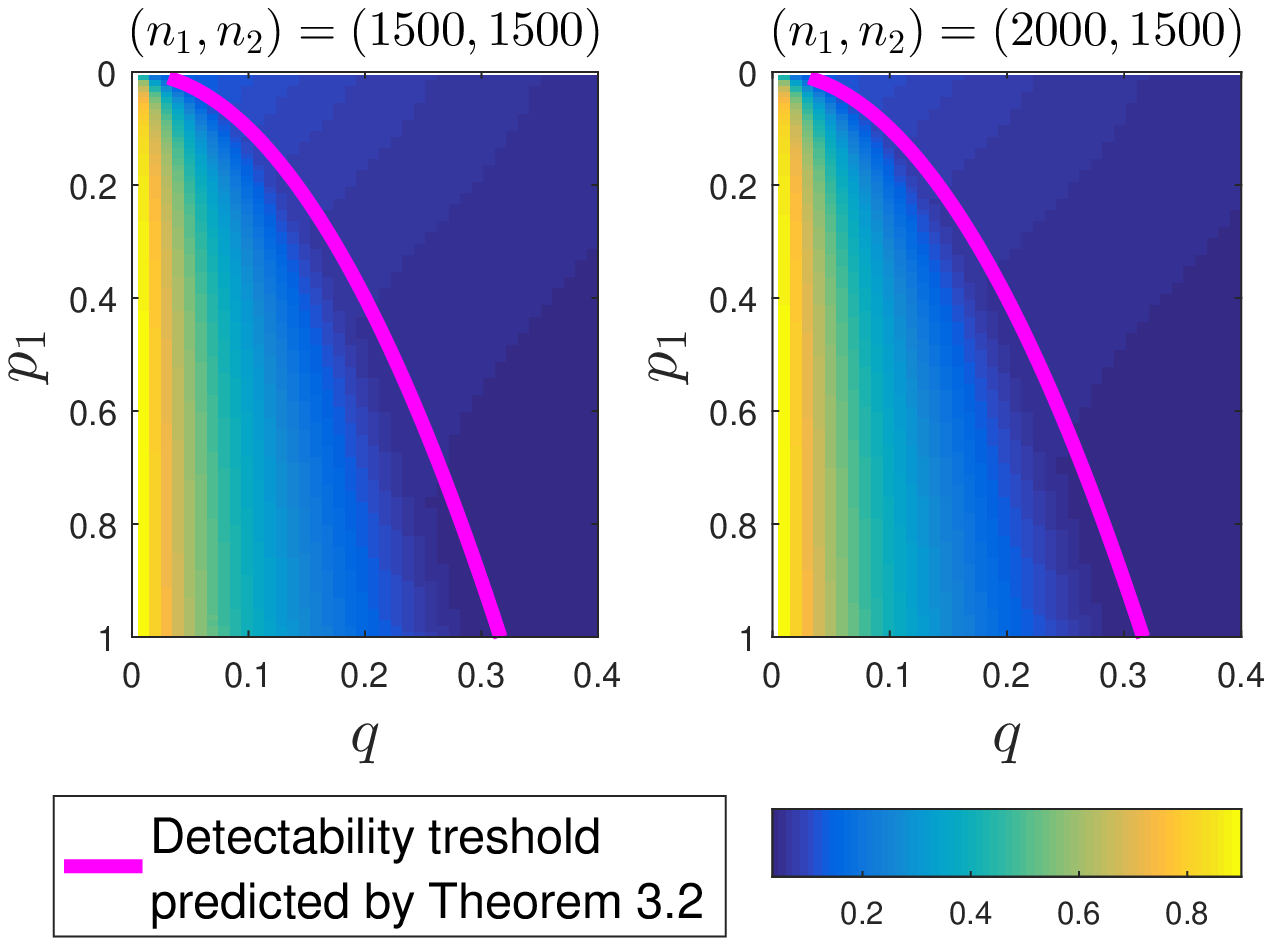}
		\vspace{-9mm}
		\caption{The value of $\theta$.}
	\end{subfigure}
	\vspace{-4mm}
	\caption{Numerical validation of community detectability via Algorithm \ref{algo_SGC} under SBM($2,\bP$) with varying $p_1$, $q$ and fixed $p_2=0.1$. Community detectability is defined as the fraction of correctly detected nodes. The pink curve specifies the theoretical detection threshold $q=\sqrt{p_1 p_2}$. When $q \geq \sqrt{p_1 p_2}$, correct community detection is impossible and $\theta$ is close to $0$, as indicated by 
		Corollary \ref{cor_SBM_two} (see (a)) and
		Theorem \ref{thm_NGL_SBM} (see (b)).}
	\label{Fig_SBM_2}
	\vspace{-4mm}
\end{figure}

Fig. \ref{Fig_SBM_2} displays two numerical examples of different community sizes to validate the detectability condition. It can be observed that in both cases when $q < \sqrt{p_1 p_2}$, correct community detection can be achieved and $\theta>0$. On the other hand, when $q \geq \sqrt{p_1 p_2}$, correct community detection is impossible and $\theta$ is close to $0$. Consequently, inspecting the data-driven parameter $\theta$ indeed reveals community detectability, which validates Theorem \ref{thm_NGL_SBM} and Corollary \ref{cor_SBM_two}.

\section{Community Detection Algorithms using SGC-GEN}
\label{sec_SGC_GEN_algo}

\subsection{SGC-GEN Meta Algorithm}
The proposed SGC-GEN framework in (\ref{eqn_SGC_GEN}) applies to any SGC method and any GCM. It is a meta algorithm that avails community detection by specifying the corresponding community detection loss function $f$ and the model mismatch metric $R$, in addition to the regularization parameter $\alpha$ and the maximum number of communities $\Kmax$. Algorithm \ref{algo_SGC_GEN} below summarizes SGC-GEN.

\begin{algorithm}
	\caption{SGC-GEN meta algorithm}
	\label{algo_SGC_GEN}
	\begin{algorithmic}
		\State 	\textbf{Input:}
		\begin{itemize}
			\item graph $\cG$
			\item spectral graph clustering (SGC) method $\cF$
			\item generative community model (GCM) $\cM$
			\item maximum number of communities $\Kmax$
			\item regularization parameter $\alpha$
			\item community detection loss function $f$	
			\item model mismatch metric $R$
		\end{itemize}
		\State  	\textbf{Output:} $K^*$ communities $\{\cG_k\}_{k=1}^{K^*}$ in $\cG$,~$2 \leq K^* \leq \Kmax$
		\State \textbf{Step 1:} Use $\cF$ to obtain the set $\cS$ of candidate community detection results.  $\cS=\{\{\cG_k\}_{k=1}^{K}: K=2,\ldots,\Kmax\}$
		\State \textbf{Step 2:} Find $K^*=\arg \min_{\{\cG_k\}_{k=1}^K \in \cS }   f(\{\cG_k\}_{k=1}^K,\cF,\cM)$ 
		\State$+ \alpha \cdot R(\{\cG_k\}_{k=1}^K,\cM)$
		\State \textbf{Step 3:} Output $\{\cG_k\}_{k=1}^{K^*}$
	\end{algorithmic}
\end{algorithm}

\begin{table*}[t]
	\centering
	\caption{Summary of 8 SGC-GEN-empowered methods (first 4 rows highlighted by brown color) and 6 comparative baseline approaches in Sec. \ref{subsec_baseline}.}
	\label{table_SGC_GEN}
	\begin{tabular}{|c|c|c|c|c|c|}
		\hline
		Method               & Algorithm         & GCM            & f   & R     & Computational complexity                                   \\ \hline
		\textcolor{brown}{SGC-EIG (regSGC-EIG)} & $\cF_1$ ($\cF_2$) & SBM($K,\bP$)   & (4) & $R_1$ & $O\lb \Kmax (m+m^{\prime \prime}) + (\Kmax^3+\Kmax) n \rb$ \\ \hline
			\textcolor{brown}{SGC-MOD (regSGC-MOD)} & $\cF_1$ ($\cF_2$) & SBM($K,\bP$)   & (4) & $R_2$ & $O\lb \Kmax m + (\Kmax^3+\Kmax) n \rb$                     \\ \hline
			\textcolor{brown}{SGC-AIC (regSGC-AIC)} & $\cF_1$ ($\cF_2$) &SBM($K,\bP$)   & (4) & $R_3$ & $O\lb \Kmax m + (\Kmax^3+\Kmax) n \rb$                     \\ \hline
			\textcolor{brown}{SGC-BIC (regSGC-BIC)} & $\cF_1$ ($\cF_2$) & SBM($K,\bP$)   & (4) & $R_4$ & $O\lb \Kmax m + (\Kmax^3+\Kmax) n \rb$                     \\ \hline
		SBM-AIC \cite{aicher2014learning}              & Bayesian inference        & SBM($K,\bP$)   & 0   & $R_3$ & $O\lb \Kmax m + (\Kmax^3+\Kmax) n \rb$   \\ \hline
		SBM-BIC \cite{aicher2014learning}                & Bayesian inference         & SBM($K,\bP$)   & 0   & $R_4$ & $O\lb \Kmax m + (\Kmax^3+\Kmax) n \rb$                     \\ \hline
		DCSBM-AIC  \cite{aicher2014learning}             & Bayesian inference       & DCSBM & 0   & $R_3$ & $O\lb \Kmax m + (\Kmax^3+\Kmax) n \rb$                     \\ \hline
		DCSBM-AIC   \cite{aicher2014learning}           & Bayesian inference         & DCSBM & 0   & $R_4$ & $O\lb \Kmax m + (\Kmax^3+\Kmax) n \rb$                     \\ \hline
		Self-Tuning  \cite{zelnik2004self}        & $\cF_1$            & None           &defined in \cite{zelnik2004self}  & 0     & $O\lb \Kmax m + (\Kmax^3+\Kmax) n \rb$                     \\ \hline
		Louvain \cite{blondel2008fast}                & Node merging      & None           & 0   & $R_2$ & $O\lb(m+n) \cdot \text{iterations} \rb$                                         \\ \hline
	\end{tabular}
\end{table*}

\subsection{SGC-GEN via $\bLN$ and SBM($K,\bP$) }
\label{subsec_SGC_GEN_eight}
Based on the theoretical analysis established in Sec. \ref{sec_SGC_GEN_thm}, here we specify 2 SGC methods, the corresponding community detection loss function, and 4 model mismatch metrics for SGC-GEN.  This yields 8 community detection methods originated from Algorithm \ref{algo_SGC_GEN}. In particular, for these methods we select the GCM $\cM$ to be SBM($K,\bP$). These SGC-GEN-empowered community detection methods are summarized in Table \ref{table_SGC_GEN}. The details are described as follows.\\
\textbf{Two SGC methods.}\\
$\bullet$ $\cF_1:$ the first method is SGC using the normalized graph Laplacian matrix $\bLN$ as described in Algorithm \ref{algo_SGC}. To obtain the set  $\cS$ in step 1 of Algorithm \ref{algo_SGC_GEN}, one computes the $\Kmax$ smallest eigenvectors of $\bLN$ and use  Algorithm \ref{algo_SGC} to obtain the candidate communities $\{\cG_k\}_{k=1}^K $ of different $K$ in $\cS$. \\
$\bullet$ $\cF_2:$ the second method is regularized SGC 	using the normalized graph Laplacian matrix $\bLN$. It is similar to $\cF_1$ except that one replaces the matrix $\bD$ in step 1 of Algorithm \ref{algo_SGC} with $\bD+ \overline{d} \bI_n$, where $\overline{d}$ is the average degree of the graph $\cG$. The regularization leads to better clustering than $\cF_1$ in sparse graphs as suggested in \cite{chaudhuri2012spectral,amini2013pseudo}.\\
\textbf{Community detection loss function.}\\
Since $\cF_1$ and $\cF_2$ are SGC methods via $\bLN$, using  Theorem \ref{thm_NGL_SBM},
we set the community detection loss function $f$ to be
\begin{align}
\label{eqn_comm_error}
f=\exp \lb - \frac{\theta}{K-1} \rb, 
\end{align}
where $\theta=\sum_{k=2}^K 1 - \lambda_k$ and $\lambda_k$ is the $k$-th smallest eigenvalue of $\bLN$. It is similar to the exponential loss function used in supervised learning problems. The denominator $K-1$ serves the purpose of comparing different community detection results in $\cS$.
When $\theta>0$, the function $f$ is confined in the interval $[0,1]$, and it favors the community detection results of small partial eigenvalue sum $\sum_{k=2}^K \lambda_k$, which is a measure of multiway cut in $\cG$ \cite{Luxburg07}.
When $\theta \leq 0$, $f$ is greater than 1 and has exponential growth as $\theta$ decreases, which implies that $f$ imposes large loss on incorrect community detection results  based on Theorem \ref{thm_NGL_SBM}. Note that $f$ is a data-driven function since it only requires the knowledge of $\{\lambda_k\}_{k=2}^K$.\\
\textbf{Four model mi$  $smatch metrics.}\\
$\bullet$ $R_1:$ spectral radius of the modular matrix  with respect to SBM($K,\bP$).  Define the $n \times n$ modular matrix $\bB$ with respect to SBM($K,\bP$) as 
$[\bB]_{ij}=
[\bA]_{ij} - \widehat{P}_{g_i g_j}$ if  $i \neq j$ and 	$[\bB]_{ij}=0$ if $i=j$,
for all $i,j \in \{1,\ldots,n\}$, where $g_i$ denotes the community label of node $i$. The parameter  $\widehat{P}_{g_i g_j}$ is the maximum likelihood estimator of $P_{g_i g_j}$ in $\bP$ given the detected communities $\{ \cG_{k}\}_{k=1}^K$, which is defined as
$\widehat{P}_{g_i g_j}=
\frac{m_{g_i g_j}}{ n_{g_i} n_{g_j} }$ if $g_i \neq g_j$ and 
$\widehat{P}_{g_i g_j}=\frac{m_{g_i g_i}}{ \binom{n_{g_i}}{2} }$ if $g_i=g_j$,	
for all $g_i,g_j \in \{1,\ldots,K\}$,
where $n_{k}$ denotes the number of nodes in $\cG_k$ and
$m_{k \ell}$ denotes the number of edges between communities $\cG_k$ and $\cG_\ell$.
$R_1$ is defined as the spectral radius of $\bB$, which is the largest eigenvalue of $\bB$ in absolute value. It relates to the first-order eigenvalue approximation of the signed triangle counts \cite{bubeck2016testing}, which is an effective statistic for testing latent structure in random graphs. \\
$\bullet$ $R_2:$ negative modularity. Given communities $\{ \cG_{k}\}_{k=1}^K$ in $\cG$, modularity is a measure of difference between $\{ \cG_{k}\}_{k=1}^K$ and a random graph of the same degree sequence \cite{Newman04mod}. The modularity is
defined as $Q=\sum_{k=1}^K (e_{kk} - b_k^2)$, where $e_{ij}=\frac{m_{ij}}{2m}$ if $i \neq j$ and $e_{ij}=\frac{m_{ii}}{m}$ if $i = j$, for all $i,j \in \{1,\ldots,K\}$, and $b_i=\sum_{j=1}^K e_{ij}$. By defining $R_2=-Q$, the model mismatch metric is small when the communities are distinct from the corresponding randomized graphs. \\
$\bullet$  $R_3:$ AIC under SBM($K,\bP$). The Akaike information criterion (AIC) is a measure of the relative quality of statistical models for a given set of data. $R_3$ is defined as the AIC given communities $\{ \cG_{k}\}_{k=1}^K$ under SBM($K,\bP$), which is $R_3=K(K-1) - 2 \phi(\{ \cG_{k}\}_{k=1}^K, SBM(K,\bP))$, where $\phi$ denotes the log-likelihood of $\{ \cG_{k}\}_{k=1}^K$ under SBM($K,\bP$). 
The closed-form expression of $\phi$ is given in \cite{Karrer11}. \\
$\bullet$  $R_4:$ BIC under SBM($K,\bP$). The  Bayesian information criterion (BIC) is another relative measure of data fitness to statistical models. $R_4$ is defined as the BIC of communities $\{ \cG_{k}\}_{k=1}^K$ under SBM($K,\bP$), which is  $R_4=\frac{\ln m}{2} \cdot K(K-1) - 2 \phi(\{ \cG_{k}\}_{k=1}^K, SBM(K,\bP))$.


\subsection{Computational Complexity Analysis}
\label{subsec_computation}
Here we analyze the computational complexity of the 8 SGC-GEN community detection methods listed in Table \ref{table_SGC_GEN}. 
There are three main factors contributing to the computational complexity: (i) computation of the $\Kmax$ smallest eigenvectors of $\bLN$, (ii) K-means clustering, and (iii) computation of the community detection loss function and the  model mismatch metric. The overall computational complexity of each method is summarized in Table \ref{table_SGC_GEN}.

For (i), computing the $\Kmax$ smallest eigenvectors of $\bLN$ requires $O(\Kmax(m+n))$ operations using power iteration techniques \cite{livne2012lean,CPY_16KDDMLG,wu2015preconditioned,wu2016estimating,wu2016primme_svds}, where $m+n$ is the number of nonzero entries in $\bLN$. For (ii), given any $K \leq \Kmax$, K-means clustering on the rows of the $K$ smallest eigenvectors of $\bLN$ requires $O(n K^2)$ operations \cite{zaki2014data}. As a result, to obtain the set $\cS$ of candidate community detection results by varying $K$ from $2$ to $\Kmax$ requires  $O(n \Kmax^3)$ operations in total. For (iii), the complexity of computing the function $\theta$ and the loss function $f$ in (\ref{eqn_comm_error}) is negligible since they can be obtained in the process of (i). The computation of $R_1$ for a given $K$ requires $O(m)$ operations for computing $\{\widehat{P}_{k \ell}\}_{k,\ell=1}^K$ and $O(m^\prime+n)$ operations for computing the spectral radius of $\bB$ using power iteration techniques, where $m^\prime$ is the number of nonzero entries in $\bB$. Therefore, the overall computational complexity of $R_1$ in SGC-GEN is $O(\Kmax(m^{\prime \prime}+m+n))$, where $m^{\prime \prime}$ is the maximum number of nonzero entries in $\bB$ ranging from $K=2$ to $K=\Kmax$. The computation of $R_2$ for a given $K$ is $O(m)$, the same complexity for computing modularity \cite{Newman04mod}. The overall computational complexity of $R_2$ in SGC-GEN is $O(\Kmax m)$. For a given $K$, the computation of $R_3$ and $R_4$ requires $O(m)$ operations to compute the closed-form log-likelihood function $\phi$. The overall computational complexity of $R_3$ and $R_4$ in SGC-GEN is $O(\Kmax m)$.  The computational complexity of $\cF_1$ and $\cF_2$ has the same order since the regularization step in $\cF_2$ simply adds $n$ entries to the degree matrix $\bD$. Similarly, the data storage of these methods require $O(\Kmax^2(m+n))$ space. 

\begin{table*}[]
	\centering
	\caption{Statistics and descriptions of the collected graph datasets. ``NA'' stands for ``not available''.}
	\label{table_dataset}
	\begin{tabular}{|c|c|c|c|c|c|c|}
		\hline
		Dataset           & Description           & Node          & Edge               & \# of nodes & \# of edges & Community labels    \\ \hline
		BlogCatalog\footnote{http://socialcomputing.asu.edu/datasets/BlogCatalog3}       & online social network & user          & friendship         & 10312       & 333983      & 39 social groups    \\ \hline
		Youtube\footnote{http://socialcomputing.asu.edu/datasets/YouTube2}           & online social network & user          & friendship         & 22180       & 96092       & 47 social groups    \\ \hline
		PoliticalBlog\footnote{http://konect.uni-koblenz.de/networks/moreno-blogs}     & online social network & user          & blog reference     & 1222        & 16714       & 2 political parties \\ \hline
		Cora\footnote{http://www.cs.umd.edu/~sen/lbc-proj/data/cora.tgz}              & publication network   & paper         & citation & 2485        & 5069        & 7 research topics   \\ \hline
		Citeseer\footnote{http://www.cs.umd.edu/~sen/lbc-proj/data/citeseer.tgz}          & publication network   & paper         & citation & 2110        & 3694        & 6 research topics   \\ \hline
		Pubmed\footnote{http://www.cs.umd.edu/projects/linqs/projects/lbc/Pubmed-Diabetes.tgz}            & publication network   & paper         & citation & 19717       & 44324       & 3 research topics   \\ \hline
		PrettyGoodPrivacy\footnote{http://konect.uni-koblenz.de/networks/arenas-pgp} & communication network & router        & connection         & 10680       & 24316       & NA                  \\ \hline
		AS-Newman\footnote{http://www-personal.umich.edu/~mejn/netdata/}         & communication network & router        & connection         & 22963       & 48436       & NA                  \\ \hline
		AS-SNAP\footnote{http://snap.stanford.edu/data/as.html}           & communication network & router        & connection         & 6474        & 12572       & NA                  \\ \hline
		Facebook\footnote{http://snap.stanford.edu/data/egonets-Facebook.html}          & online social network & user          & friendship         & 4039        & 88234       & NA                  \\ \hline
		Email-Arenas\footnote{http://konect.uni-koblenz.de/networks/arenas-email}      & email network         & user          & communication      & 1133        & 5451        & NA                  \\ \hline
		Email-Enron\footnote{http://snap.stanford.edu/data/email-Enron.html}       & email network         & user          & communication      & 33696       & 180811      & NA                  \\ \hline
		MinnesotaRoad\footnote{http://www.cise.ufl.edu/research/sparse/matrices/Gleich/minnesota.html}     & physical network      & intersection  & road               & 2640        & 3302        & NA                  \\ \hline
		PowerGrid\footnote{http://konect.uni-koblenz.de/networks/opsahl-powergrid}         & physical network      & power station & power line         & 4941        & 6594        & NA                  \\ \hline
		Reactome\footnote{http://konect.uni-koblenz.de/networks/reactome}          & biological network    & protein       & interaction        & 5973        & 146385      & NA                  \\ \hline
		CAAstroPh\footnote{http://snap.stanford.edu/data/ca-AstroPh.html}         & collaboration network & researcher    & coauthorship       & 17903       & 197000      & NA                  \\ \hline
		CAHepPh\footnote{http://snap.stanford.edu/data/ca-HepPh.html}           & collaboration network & researcher    & coauthorship       & 21363       & 91314       & NA                  \\ \hline
		CACondMat\footnote{http://snap.stanford.edu/data/ca-CondMat.html}         & collaboration network & researcher    & coauthorship       & 11204       & 117634      & NA                  \\ \hline
	\end{tabular}
\end{table*}

In summary, the overall computational complexity of SGC-GEN-enabled methods is linear in the number of nodes and edges ($n$ and $m$) and depends on $\Kmax$. In practice $\Kmax$ is a constant such that $\Kmax \ll$ $n$ and $m$. Based on the computational analysis, the community detection methods based on SGC-GEN have the same order of complexity in $n$ and $m$ when compared with the baseline methods of similar objective functions described in Sec. \ref{subsec_baseline}, which suggests that utilizing SGC-GEN for community detection is computationally as efficient as these baseline methods.

\section{Performance Evaluation}
\label{sec_performance}

\subsection{Dataset Description and Evaluation Metrics}
\label{subsec_data_des}
\textbf{Dataset Description.} To compare the performance of community detection, we collected 18 real-life graph datasets from various domains, including online social, physical, biological, communication, collaboration, email, and publication networks. For each dataset, we extracted the largest connected component as the input graph $\cG$ for community detection. All input graphs are made undirected, unweighted and unlabeled.
Among these datasets, 6 datasets are provided with additional community labels. If a node in the graph is provided with more than one community label, the most common label among its neighboring nodes is assigned to the node.
The statistics of the collected graphs are summarized in Table \ref{table_dataset}.

\textbf{Evaluation metrics.} We use 7 representative external and internal clustering metrics to evaluate the performance of different communication detection methods. External clustering metrics can  be computed when the community labels are given. Internal clustering metrics evaluate the quality of communities in terms of connectivity, which can be computed without community labels. \\
\textbf{external clustering metrics:}\\
$\bullet$ Normalized mutual information (NMI) \cite{zaki2014data}.	\\
$\bullet$ Rand index (RI) \cite{zaki2014data}.	\\
$\bullet$ F-measure (FM) \cite{zaki2014data}.  \\
These external clustering metrics are properly scaled between 0 and 1, and larger value means better clustering  performance. \\
\textbf{internal clustering metrics:} \\
$\bullet$ Conductance (COND) \cite{Shi00}: the averaged COND over all communities. Lower value means better performance. \\
$\bullet$  Normalized cut (NC) \cite{Shi00}: the averaged NC over all communities. Lower value means better performance. \\
$\bullet$ Average out-degree fraction (avg-ODF) \cite{flake2000efficient}: the averaged avg-ODF over all communities. Lower value means better performance.\\
$\bullet$ Modularity (MOD) \cite{Newman04mod}: MOD is defined in the model mismatch metric $R_2$ in Sec. \ref{subsec_SGC_GEN_eight}. Larger value means better  performance.

\footnotetext[2]{http://socialcomputing.asu.edu/datasets/BlogCatalog3} 
\footnotetext[3]{http://socialcomputing.asu.edu/datasets/YouTube2} 
\footnotetext[4]{http://konect.uni-koblenz.de/networks/moreno-blogs}
\footnotetext[5]{http://www.cs.umd.edu/~sen/lbc-proj/data/cora.tgz}
\footnotetext[6]{http://www.cs.umd.edu/~sen/lbc-proj/data/citeseer.tgz}
\footnotetext[7]{http://www.cs.umd.edu/projects/linqs/projects/lbc/Pubmed-Diabetes.tgz} \footnotetext[8]{http://konect.uni-koblenz.de/networks/arenas-pgp} 
\footnotetext[9]{http://www-personal.umich.edu/~mejn/netdata/}
\footnotetext[10]{http://snap.stanford.edu/data/as.html}
\footnotetext[11]{http://snap.stanford.edu/data/egonets-Facebook.html} 
\footnotetext[12]{http://konect.uni-koblenz.de/networks/arenas-email} 
\footnotetext[13]{http://snap.stanford.edu/data/email-Enron.html}
\footnotetext[14]{http://www.cise.ufl.edu/research/sparse/matrices/Gleich/minnesota.html}
\footnotetext[15]{http://konect.uni-koblenz.de/networks/opsahl-powergrid}
\footnotetext[16]{http://konect.uni-koblenz.de/networks/reactome} \footnotetext[17]{http://snap.stanford.edu/data/ca-AstroPh.html}
\footnotetext[18]{http://snap.stanford.edu/data/ca-HepPh.html}
\footnotetext[19]{http://snap.stanford.edu/data/ca-CondMat.html}

\textbf{Average rank score.}
To combine multiple clustering metrics for performance evaluation of different community detection methods, we adopt the methodology proposed in \cite{leskovec2010empirical,yang2015defining} and use the average rank score of all clustering metrics as the performance metric. For each dataset, we rank each community detection method for every clustering metric via standard competition rankings  and obtain an average rank score of all clustering metrics. Therefore,  lower average rank score means better  community detection.

\subsection{Baseline Comparative Methods}
\label{subsec_baseline}
As summarized in Table \ref{table_SGC_GEN},
we compare the performance of SGC-GEN methods with 6 baseline community detection methods of similar loss functions and model mismatch metrics: \\
$\bullet$ SBM-AIC: Given the number $K$ of communities, SBM-AIC uses Bayesian inference techniques to evaluate the posterior distribution of community assignments given the graph $\cG$ under the SBM. We implemented the state-of-the-art package WSBM\footnote{http://tuvalu.santafe.edu/~aaronc/wsbm/}  to obtain the mostly probable communities $\{\cG_k\}_{k=1}^K$ \cite{aicher2014learning} and use the AIC to determine the final  communities ranging from $K=2$ to $K=\Kmax$.\\
$\bullet$ SBM-BIC: SBM-BIC is the same as SBM-AIC except that one uses the BIC  to determine the final community detection results. \\
$\bullet$ DCSBM-AIC:  DCSBM-AIC is the same as SBM-AIC except that one uses the degree-corrected SBM (DCSBM) \cite{Karrer11} for inference. \\
$\bullet$ DCSBM-BIC: DCSBM-BIC is the same as SBM-BIC except that one uses the degree-corrected SBM (DCSBM) \cite{Karrer11} for inference. \\
$\bullet$ Self-Tuning\footnote{http://www.vision.caltech.edu/lihi/Demos/SelfTuningClustering.html}:  Self-Tuning is a SGC algorithm that uses an energy function based on  $\bLN$ for basis rotation and finds the best community detection results among $2$ to $\Kmax$ communities \cite{zelnik2004self}. \\
$\bullet$ Louvain\footnote{https://perso.uclouvain.be/vincent.blondel/research/louvain.html}: Louvain method is a greedy modularity maximization approach for community detection based on node merging \cite{blondel2008fast}.

\subsection{The Effect of Regularization Parameter $\alpha$}
Here we investigate the effect of the regularization parameter $\alpha$ in (\ref{eqn_SGC_GEN}) on the performance of the eight SGC-GEN community detection methods listed in Table \ref{table_SGC_GEN}.  We set $\Kmax=50$ and use the six datasets with additional community labels in Table \ref{table_dataset} to select $\alpha$ from the set $\{0,10^{-6},10^{-5},\ldots,10^2\}$. For illustration, Fig. \ref{Fig_rank_plot} displays the stacked average rank plot of these SGC-GEN methods separately ranked by different $\alpha$ in Youtube and Citeseer datasets. The colors represent different methods and the width of each colored block represents average rank score based on the selected values of $\alpha$.
It is observed that for each method, setting large $\alpha$ (i.e., underestimating the loss function) or neglecting the model mismatch metric (i.e., setting $\alpha=0$) leads to the worst performance, which justifies the motivation of SGC-GEN. In addition, sweeping $\alpha$ within $\{10^{-6},\ldots,10^{-1}\}$ does not induce drastic changes in the average rank score, which demonstrates the robustness of SGC-GEN.
Based on the average rank score of these datasets, for the following experiments we assign $\alpha=10^{-4}$  ($\alpha=10^{-6}$)
to the (regularized) SGC-GEN methods.

\begin{figure}[t]
	\hspace{-3mm}
	\begin{subfigure}[b]{0.54\linewidth}
		\includegraphics[width=\textwidth]{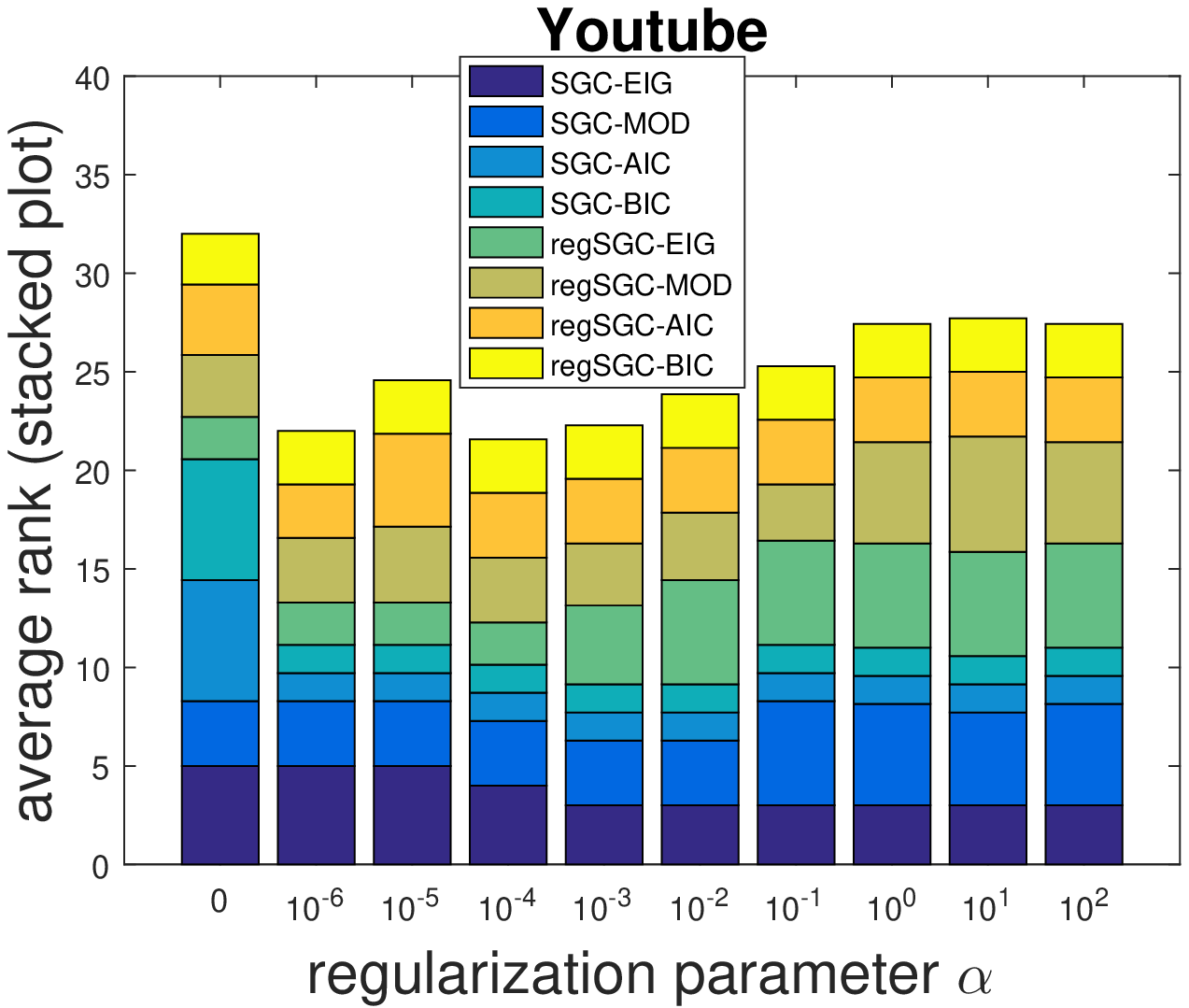}
	\end{subfigure}%
	\begin{subfigure}[b]{0.54\linewidth}
		\includegraphics[width=\textwidth]{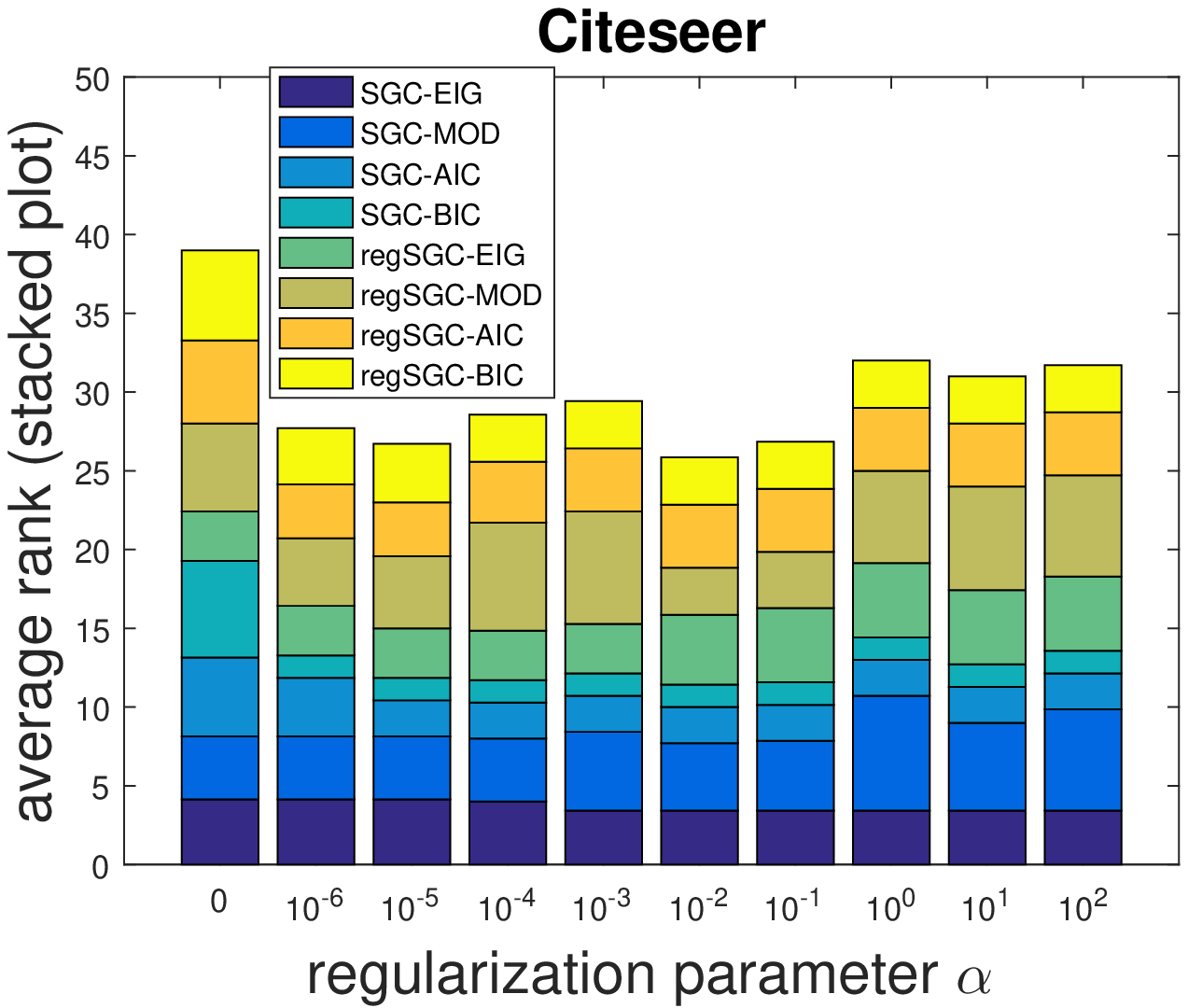}
	\end{subfigure}
	\caption{Average rank score of the 8 SGC-GEN community detection methods in Table \ref{table_SGC_GEN} with respect to different regularization parameter $\alpha$ of Youtube and Citeseer datasets. For each method, thicker block means the corresponding $\alpha$ value leads to worse clustering performance.
		Setting large $\alpha$ or neglecting the model mismatch metric ($\alpha=0$) yields poor performance. }
	\label{Fig_rank_plot}
\end{figure}

\subsection{Comparison to Baseline Methods}
Here we compare the 8 SGC-GEN methods to the 6 baseline methods  in Sec. \ref{subsec_baseline}.
For the Bayesian inference baseline methods we set $\Kmax=20$, since we observe that larger $\Kmax$ does not improve their performance but significantly increases the computation time.
For the SGC-GEN methods and Self-Tuning we set $\Kmax=50$. For Louvain one does not need to specify $\Kmax$.
All experiments are implemented by Matlab R2016 on a 16-core cluster with 128 GB RAM. 



Table \ref{table_avg_rank} displays the mean and standard deviation of average rank scores over all 18 graph datasets for each community detection method. Among these 14 methods, SGC-MOD and SGC-EIG have the best and second best mean average rank score over all datasets, which suggests that joint consideration of theoretical detectability and modular structure using the proposed SGC-GEN framework improves community detection.
The results also suggest that the degree regularization technique does not necessarily guarantee better performance.
For the baseline methods, it can be observed that Bayesian inference based approaches  lead to poor performance, which can be explained by the fact the graph datasets may not comply with the assumption of the underlying generative community models. Louvain also yields poor performance since it is a greedy algorithm that only aims to maximize one single clustering metric (i.e., modularity). 
Self-Tuning performs better than some SGC-GEN methods  but it does not prevail SGC-MOD and SGC-EIG, which can be explained by the fact that the energy function used in Self-Tuning does not exploit the discriminative power of community detectability. Since in Sec. \ref{subsec_computation} SGC-GEN is shown to be computationally as efficient as these baseline methods, we conclude that community detection via SGC-GEN yields superior performance without incurring additional computational costs.

\begin{table}[t]
	\centering
	\caption{Performance evaluation of 14 community detection methods. Lower average rank score means better performance. The 8 SGC-GEN-based methods are highlighted by brown color. The proposed SGC-MOD and SGC-EIG achieve the best and second best performance, respectively.}
	\label{table_avg_rank}
	\begin{tabular}{|c|c|c|}
		\hline
		\multirow{2}{*}{Method} & \multicolumn{2}{c|}{Average rank of all datasets} \\ \cline{2-3} 
		& mean               & standard deviation           \\ \hline
		\textcolor{brown}{SGC-EIG}                 & 4.6290             & 2.0178                       \\ \hline
		\textcolor{brown}{SGC-MOD}                 & 4.3433             & 1.3484                       \\ \hline
		\textcolor{brown}{SGC-AIC}                 & 5.5476             & 2.1481                       \\ \hline
		\textcolor{brown}{SGC-BIC}                 & 5.1468             & 1.6762                       \\ \hline
		\textcolor{brown}{regSGC-EIG}              & 6.0417             & 1.3403                       \\ \hline
		\textcolor{brown}{regSGC-MOD}             & 5.3313             & 1.3592                       \\ \hline
		\textcolor{brown}{regSGC-AIC}              & 6.1409             & 1.8277                       \\ \hline
		\textcolor{brown}{regSGC-BIC}              & 6.0298             & 1.8559                       \\ \hline
		SBM-AIC                 & 10.9385            & 1.2383                       \\ \hline
		SBM-BIC                 & 10.9385            & 1.2383                       \\ \hline
		DCSBM-AIC               & 11.5675            & 1.3735                       \\ \hline
		DCSBM-BIC               & 11.5675            & 1.3735                       \\ \hline
		Self-Tuing              & 4.9821             & 1.1923                       \\ \hline
		Louvain                 & 6.1250             & 2.2236                       \\ \hline
	\end{tabular}
\end{table}

\subsection{Comparison in graph domains and types}
For further analysis, we categorize the 18 graph datasets in Table \ref{table_dataset} into 7 domains based on their descriptions.  Fig. \ref{Fig_domain} displays the mean average rank score of each domain for 10 selected methods. It can be observed that no single community detection method  outperforms others in all domains. For example, regSGC-MOD has superior performance in online social, publication and biological networks but has poor performance in email and communication networks. SBM-BIC has the best performance in communication networks but not in other domains.  SGC-MOD has the best averaged performance over all datasets but it does not prevail others in every domain.
The
results suggest that considering the graph domain is essential for improving community detection.

We also separate the 18 datasets into two types: \textit{with community labels} or \textit{without community labels}. The corresponding average rank score is shown in Fig. \ref{Fig_type}. For the datasets with community labels, regSGC-MOD, regSGC-AIC and regSGC-BIC are outstanding, whereas for the datasets without community labels SGC-EIG and SGC-MOD prevail. Since community labels provide additional external clustering metrics, the results suggest that the  external and internal clustering metrics have different evaluation criterion.

\begin{figure}[!t]
	\hspace{-7mm}
	\includegraphics[width=4in]{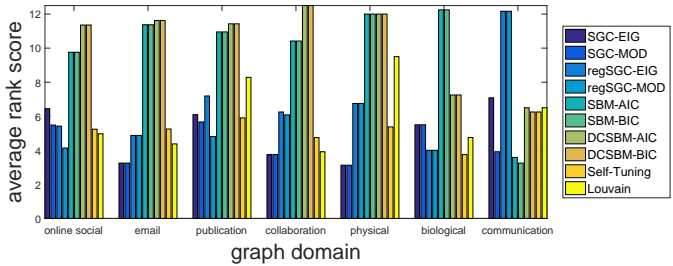}
	\caption{The mean of domain-wise average rank score. Although SGC-GEN yields the best overall performance, no single method  outperforms others in all domains.}
	\label{Fig_domain}
\end{figure}

\begin{figure}[!t]
	\centering
	\includegraphics[width=3.6in]{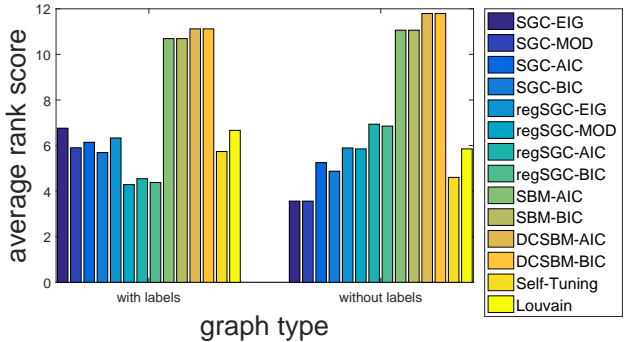}
	\caption{The mean of type-wise average rank score. The difference suggests that the external and internal clustering metrics  have different evaluation criterion.  }
	\label{Fig_type}
\end{figure}

\section{Related Work}
\label{sec_related}
Community detection and graph clustering have been an active research field in the past two decades. We refer readers to \cite{Fortunato10,fortunato2016community} for an overview  of community detection methods. In recent years, there has been a major breakthrough in analyzing both the informational and algorithmic limits of community detection under certain generative community models (GCMs). In this section we summarize the recent research findings in  community detectability. 

Many informational and algorithmic limits of community detection have been analyzed under the stochastic block model (SBM) \cite{Holland83}. Abbe et al. analyzed the informational limit by specifying the detectable and undetectable regimes for community detection via the parameters of SBM \cite{abbe2015community}. They also proposed a belief propagation algorithm that is proved to achieve the informational limit \cite{abbe2015detection}.
Hajek et al. proposed a semidefinite programming algorithm that achieves the informational limit \cite{hajek2016achieving,hajek2016achieving_ext}. 
Inference approaches based on statistical physics have been studied in \cite{Decelle11,Krzakala2013}. Spectral graph clustering algorithms, including the modularity matrix, the graph Laplacian matrix, the adjacency matrix, and the modular matrix, have been studied in \cite{rohe2011spectral,Nadakuditi12Detecability,chaudhuri2012spectral,Peixoto13,Radicchi13_hetero,CPY14modularity,saade2014spectral,Lei15,le2015concentration,joseph2016impact} and applied to various applications in graph mining \cite{CPY14deep,zhang2017name,CPY16ICASSP_short,zhang2014name,saha2015name,CPY14ComMag,dundar2015simplicity,CPY14ICASSP_short} and machine learning  \cite{peng2016recurrent,CPY17_Laplacian_tradeoff,peng2015circle,liu2017accelerated}.

Beyond the SBM, Zhao et al. proved the consistency of community detection \cite{Zhao12} under the degree-corrected SBM \cite{Karrer11}. Under the same model, Qin and Rohe studied regularized spectral clustering \cite{qin2013regularized}, and Gao et al. derived a minimax risk \cite{gao2016community}.
Chen and Hero proved the algorithmic limit of spectral clustering \cite{CPY14spectral,CPY16AMOS,CPY17MIMOSA} under the random interconnection model \cite{CPY16AMOS}. Although studying the limits of community detection methods under GCMs provides novel insights on evaluating community detectability, these approaches assume the graphs are consistent with the underlying GCMs and therefore neglect the error induced by model mismatch, which motivates the SGC-GEN framework proposed in this paper.

\section{Conclusion and Future Work}
\label{sec_conclusion}
In this paper we propose SGC-GEN, a new community detection framework that jointly exploits the discriminative power of community detectability under generative community models and confines the corresponding model mismatch. A novel condition on correct community detection is 
established for SGC-GEN, leading to effective and computationally efficient community detection methods.

Performance evaluation on 18 datasets and 7 clustering metrics shows that joint consideration of community detectability and modular structure via SGC-GEN outperforms 6 baseline approaches in terms of the average rank score. We also investigated the effect of graph domains and graph types on community detection.

The performance analysis  established  in this paper focuses on the standard formulation of SGC rooted in many advanced methods.
Our future work involves developing scalable implementation of SGC-GEN to efficiently handle large-scale graphs and extending  SGC-GEN to advanced community detection methods and models.

%
\IEEEpeerreviewmaketitle



%

\bibliographystyle{IEEEtran}
\bibliography{IEEEabrv20160824,CPY_ref_20170912}

\begin{thebibliography}{10}
\providecommand{\url}[1]{#1}
\csname url@samestyle\endcsname
\providecommand{\newblock}{\relax}
\providecommand{\bibinfo}[2]{#2}
\providecommand{\BIBentrySTDinterwordspacing}{\spaceskip=0pt\relax}
\providecommand{\BIBentryALTinterwordstretchfactor}{4}
\providecommand{\BIBentryALTinterwordspacing}{\spaceskip=\fontdimen2\font plus
\BIBentryALTinterwordstretchfactor\fontdimen3\font minus
  \fontdimen4\font\relax}
\providecommand{\BIBforeignlanguage}[2]{{%
\expandafter\ifx\csname l@#1\endcsname\relax
\typeout{** WARNING: IEEEtran.bst: No hyphenation pattern has been}%
\typeout{** loaded for the language `#1'. Using the pattern for}%
\typeout{** the default language instead.}%
\else
\language=\csname l@#1\endcsname
\fi
#2}}
\providecommand{\BIBdecl}{\relax}
\BIBdecl

\bibitem{Fortunato10}
S.~Fortunato, ``Community detection in graphs,'' \emph{Physics Reports}, vol.
  486, no. 3-5, pp. 75--174, 2010.

\bibitem{Luxburg07}
U.~Luxburg, ``A tutorial on spectral clustering,'' \emph{Statistics and
  Computing}, vol.~17, no.~4, pp. 395--416, Dec. 2007.

\bibitem{Newman06community}
M.~E.~J. Newman, ``Finding community structure in networks using the
  eigenvectors of matrices,'' \emph{Phys. Rev. E}, vol.~74, p. 036104, Sep
  2006.

\bibitem{White05}
S.~White and P.~Smyth, ``A spectral clustering approach to finding communities
  in graph.'' in \emph{SIAM International Conference on Data Mining (SDM)},
  vol.~5, 2005, pp. 76--84.

\bibitem{Newman04mod}
M.~E.~J. Newman, ``Fast algorithm for detecting community structure in
  networks,'' \emph{Phys. Rev. E}, vol.~69, p. 066133, Jun 2004.

\bibitem{leskovec2010empirical}
J.~Leskovec, K.~J. Lang, and M.~Mahoney, ``Empirical comparison of algorithms
  for network community detection,'' in \emph{ACM International Conference on
  World Wide Web (WWW)}, 2010, pp. 631--640.

\bibitem{yang2015defining}
J.~Yang and J.~Leskovec, ``Defining and evaluating network communities based on
  ground-truth,'' \emph{Knowledge and Information Systems}, vol.~42, no.~1, pp.
  181--213, 2015.

\bibitem{Shi00}
J.~Shi and J.~Malik, ``Normalized cuts and image segmentation,'' \emph{{IEEE}
  Trans. Pattern Anal. Mach. Intell.}, vol.~22, no.~8, pp. 888--905, 2000.

\bibitem{ng2002spectral}
A.~Y. Ng, M.~I. Jordan, and Y.~Weiss, ``On spectral clustering: Analysis and an
  algorithm,'' in \emph{Advances in neural information processing systems
  (NIPS)}, 2002, pp. 849--856.

\bibitem{Holland83}
P.~W. Holland, K.~B. Laskey, and S.~Leinhardt, ``Stochastic blockmodels: First
  steps,'' \emph{Social Networks}, vol.~5, no.~2, pp. 109--137, 1983.

\bibitem{liu2013large}
J.~Liu, C.~Wang, M.~Danilevsky, and J.~Han, ``Large-scale spectral clustering
  on graphs,'' in \emph{International Joint Conference on Artificial
  Intelligence}.\hskip 1em plus 0.5em minus 0.4em\relax AAAI Press, 2013, pp.
  1486--1492.

\bibitem{nie2014clustering}
F.~Nie, X.~Wang, and H.~Huang, ``Clustering and projected clustering with
  adaptive neighbors,'' in \emph{ACM International Conference on Knowledge
  Discovery and Data Mining (KDD)}, 2014, pp. 977--986.

\bibitem{li2016scalable}
Y.~Li, J.~Huang, and W.~Liu, ``Scalable sequential spectral clustering.'' in
  \emph{AAAI}, 2016, pp. 1809--1815.

\bibitem{nie2016constrained}
F.~Nie, X.~Wang, M.~I. Jordan, and H.~Huang, ``The constrained laplacian rank
  algorithm for graph-based clustering.'' in \emph{AAAI}, 2016, pp. 1969--1976.

\bibitem{Goldenberg2010survey}
A.~Goldenberg, A.~X. Zheng, S.~E. Fienberg, and E.~M. Airoldi, ``A survey of
  statistical network models,'' \emph{Foundations and Trends{\textregistered}
  in Machine Learning}, vol.~2, no.~2, pp. 129--233, 2010.

\bibitem{diaconis2007graph}
P.~Diaconis and S.~Janson, ``Graph limits and exchangeable random graphs,''
  \emph{arXiv preprint arXiv:0712.2749}, 2007.

\bibitem{zhang2015estimating}
Y.~Zhang, E.~Levina, and J.~Zhu, ``Estimating network edge probabilities by
  neighborhood smoothing,'' \emph{arXiv preprint arXiv:1509.08588}, 2015.

\bibitem{abbe2016exact}
E.~Abbe, A.~S. Bandeira, and G.~Hall, ``Exact recovery in the stochastic block
  model,'' \emph{{IEEE} Trans. Inf. Theory}, vol.~62, no.~1, pp. 47--487, 2016.

\bibitem{zelnik2004self}
L.~Zelnik-Manor and P.~Perona, ``Self-tuning spectral clustering,'' in
  \emph{Advances in neural information processing systems (NIPS)}, 2004, pp.
  1601--1608.

\bibitem{blondel2008fast}
V.~D. Blondel, J.-L. Guillaume, R.~Lambiotte, and E.~Lefebvre, ``Fast unfolding
  of communities in large networks,'' \emph{Journal of Statistical Mechanics:
  Theory and Experiment}, no.~10, 2008.

\bibitem{Newman06PNAS}
M.~E.~J. Newman, ``Modularity and community structure in networks,''
  \emph{Proc. National Academy of Sciences}, vol. 103, no.~23, pp. 8577--8582,
  2006.

\bibitem{Karrer11}
B.~Karrer and M.~E.~J. Newman, ``Stochastic blockmodels and community structure
  in networks,'' \emph{Phys. Rev. E}, vol.~83, p. 016107, Jan 2011.

\bibitem{CPY16AMOS}
P.-Y. Chen and A.~O. Hero, ``Phase transitions and a model order selection
  criterion for spectral graph clustering,'' \emph{arXiv preprint
  arXiv:1604.03159}, 2016.

\bibitem{Latala05}
R.~Latala, ``{Some estimates of norms of random matrices.}'' \emph{Proc. Am.
  Math. Soc.}, vol. 133, no.~5, pp. 1273--1282, 2005.

\bibitem{Talagrand95}
M.~Talagrand, ``Concentration of measure and isoperimetric inequalities in
  product spaces,'' \emph{Publications Mathématiques de l'Institut des Hautes
  Études Scientifiques}, vol.~81, no.~1, pp. 73--205, 1995.

\bibitem{le2015concentration}
C.~M. Le and R.~Vershynin, ``Concentration and regularization of random
  graphs,'' \emph{arXiv preprint arXiv:1506.00669}, 2015.

\bibitem{joseph2016impact}
A.~Joseph, B.~Yu \emph{et~al.}, ``Impact of regularization on spectral
  clustering,'' \emph{The Annals of Statistics}, vol.~44, no.~4, pp.
  1765--1791, 2016.

\bibitem{CPY14modularity}
P.-Y. Chen and A.~O. Hero, ``Universal phase transition in community
  detectability under a stochastic block model,'' \emph{Phys. Rev. E}, vol.~91,
  p. 032804, Mar 2015.

\bibitem{Peixoto13}
T.~P. Peixoto, ``Eigenvalue spectra of modular networks,'' \emph{Phys. Rev.
  Lett.}, vol. 111, p. 098701, Aug 2013.

\bibitem{Zhao12}
Y.~Zhao, E.~Levina, and J.~Zhu, ``Consistency of community detection in
  networks under degree-corrected stochastic block models,'' \emph{The Annals
  of Statistics}, vol.~40, no.~4, pp. 2266--2292, 08 2012.

\bibitem{aicher2014learning}
C.~Aicher, A.~Z. Jacobs, and A.~Clauset, ``Learning latent block structure in
  weighted networks,'' \emph{Journal of Complex Networks}, p. cnu026, 2014.

\bibitem{chaudhuri2012spectral}
K.~Chaudhuri, F.~C. Graham, and A.~Tsiatas, ``Spectral clustering of graphs
  with general degrees in the extended planted partition model,'' in
  \emph{COLT}, vol.~23, 2012, pp. 35--1.

\bibitem{amini2013pseudo}
A.~A. Amini, A.~Chen, P.~J. Bickel, E.~Levina \emph{et~al.},
  ``Pseudo-likelihood methods for community detection in large sparse
  networks,'' \emph{The Annals of Statistics}, vol.~41, no.~4, pp. 2097--2122,
  2013.

\bibitem{bubeck2016testing}
S.~Bubeck, J.~Ding, R.~Eldan, and M.~Z. R{\'a}cz, ``Testing for
  high-dimensional geometry in random graphs,'' \emph{Random Structures \&
  Algorithms}, 2016.

\bibitem{livne2012lean}
O.~E. Livne and A.~Brandt, ``Lean algebraic multigrid (lamg): Fast graph
  {Laplacian} linear solver,'' \emph{SIAM Journal on Scientific Computing},
  vol.~34, no.~4, pp. B499--B522, 2012.

\bibitem{CPY_16KDDMLG}
P.-Y. Chen, B.~Zhang, M.~A. Hasan, and A.~O. Hero, ``Incremental method for
  spectral clustering of increasing orders,'' in \emph{ACM International
  Conference on Knowledge Discovery and Data Mining (KDD) Workshop on Mining
  and Learning with Graphs}, 2016, arXiv preprint arXiv:1512.07349.

\bibitem{wu2015preconditioned}
L.~Wu and A.~Stathopoulos, ``A preconditioned hybrid svd method for accurately
  computing singular triplets of large matrices,'' \emph{SIAM Journal on
  Scientific Computing}, vol.~37, no.~5, pp. S365--S388, 2015.

\bibitem{wu2016estimating}
L.~Wu, J.~Laeuchli, V.~Kalantzis, A.~Stathopoulos, and E.~Gallopoulos,
  ``Estimating the trace of the matrix inverse by interpolating from the
  diagonal of an approximate inverse,'' \emph{Journal of Computational
  Physics}, vol. 326, pp. 828--844, 2016.

\bibitem{wu2016primme_svds}
L.~Wu, E.~Romero, and A.~Stathopoulos, ``Primme\_{SVDS}: A high-performance
  preconditioned svd solver for accurate large-scale computations,''
  \emph{arXiv preprint arXiv:1607.01404}, 2016.

\bibitem{zaki2014data}
M.~J. Zaki and W.~Meira~Jr, \emph{Data mining and analysis: fundamental
  concepts and algorithms}.\hskip 1em plus 0.5em minus 0.4em\relax Cambridge
  University Press, 2014.

\bibitem{flake2000efficient}
G.~W. Flake, S.~Lawrence, and C.~L. Giles, ``Efficient identification of web
  communities,'' in \emph{ACM International Conference on Knowledge Discovery
  and Data Mining (KDD)}, 2000, pp. 150--160.

\bibitem{fortunato2016community}
S.~Fortunato and D.~Hric, ``Community detection in networks: A user guide,''
  \emph{Physics Reports}, vol. 659, pp. 1--44, 2016.

\bibitem{abbe2015community}
E.~Abbe and C.~Sandon, ``Community detection in general stochastic block
  models: fundamental limits and efficient recovery algorithms,'' \emph{arXiv
  preprint arXiv:1503.00609}, 2015.

\bibitem{abbe2015detection}
------, ``Detection in the stochastic block model with multiple clusters: proof
  of the achievability conjectures, acyclic bp, and the information-computation
  gap,'' \emph{Advances in Neural Information Processing Systems (NIPS)}, 2016.

\bibitem{hajek2016achieving}
B.~Hajek, Y.~Wu, and J.~Xu, ``Achieving exact cluster recovery threshold via
  semidefinite programming,'' \emph{{IEEE} Trans. Inf. Theory}, vol.~62, no.~5,
  pp. 2788--2797, 2016.

\bibitem{hajek2016achieving_ext}
------, ``Achieving exact cluster recovery threshold via semidefinite
  programming: Extensions,'' \emph{{IEEE} Trans. Inf. Theory}, vol.~62, no.~10,
  pp. 5918--5937, 2016.

\bibitem{Decelle11}
A.~Decelle, F.~Krzakala, C.~Moore, and L.~Zdeborov\'a, ``Inference and phase
  transitions in the detection of modules in sparse networks,'' \emph{Phys.
  Rev. Lett.}, vol. 107, p. 065701, Aug 2011.

\bibitem{Krzakala2013}
F.~Krzakala, C.~Moore, E.~Mossel, J.~Neeman, A.~Sly, L.~Zdeborova, and
  P.~Zhang, ``Spectral redemption in clustering sparse networks,'' \emph{Proc.
  National Academy of Sciences}, vol. 110, pp. 20\,935--20\,940, 2013.

\bibitem{rohe2011spectral}
K.~Rohe, S.~Chatterjee, and B.~Yu, ``Spectral clustering and the
  high-dimensional stochastic blockmodel,'' \emph{The Annals of Statistics},
  pp. 1878--1915, 2011.

\bibitem{Nadakuditi12Detecability}
R.~R. Nadakuditi and M.~E.~J. Newman, ``Graph spectra and the detectability of
  community structure in networks,'' \emph{Phys. Rev. Lett.}, vol. 108, p.
  188701, May 2012.

\bibitem{Radicchi13_hetero}
F.~Radicchi, ``Detectability of communities in heterogeneous networks,''
  \emph{Phys. Rev. E}, vol.~88, p. 010801, Jul 2013.

\bibitem{saade2014spectral}
A.~Saade, F.~Krzakala, and L.~Zdeborov{\'a}, ``Spectral clustering of graphs
  with the bethe hessian,'' in \emph{Advances in neural information processing
  systems (NIPS)}, 2014, pp. 406--414.

\bibitem{Lei15}
J.~Lei and A.~Rinaldo, ``Consistency of spectral clustering in stochastic block
  models,'' \emph{Ann. Statist.}, vol.~43, no.~1, pp. 215--237, 02 2015.

\bibitem{CPY14deep}
P.-Y. Chen and A.~Hero, ``Deep community detection,'' \emph{{IEEE} Trans.
  Signal Process.}, vol.~63, no.~21, pp. 5706--5719, Nov. 2015.

\bibitem{zhang2017name}
B.~Zhang and M.~A. Hasan, ``Name disambiguation in anonymized graphs using
  network embedding,'' in \emph{CIKM}, 2017.

\bibitem{CPY16ICASSP_short}
P.-Y. Chen, S.~Choudhury, and A.~O. Hero, ``Multi-centrality graph spectral
  decompositions and their application to cyber intrusion detection,'' in
  \emph{IEEE ICASSP}, 2016, pp. 4553--4557.

\bibitem{zhang2014name}
B.~Zhang, T.~K. Saha, and M.~Al~Hasan, ``Name disambiguation from link data in
  a collaboration graph,'' in \emph{IEEE/ACM ASONAM}, 2014, pp. 81--84.

\bibitem{saha2015name}
T.~K. Saha, B.~Zhang, and M.~Al~Hasan, ``Name disambiguation from link data in
  a collaboration graph using temporal and topological features,'' \emph{Social
  Network Analysis and Mining}, vol.~5, no.~1, p.~11, 2015.

\bibitem{CPY14ComMag}
P.-Y. Chen and A.~O. Hero, ``Assessing and safeguarding network resilience to
  nodal attacks,'' \emph{{IEEE} Commun. Mag.}, vol.~52, no.~11, pp. 138--143,
  Nov. 2014.

\bibitem{dundar2015simplicity}
M.~Dundar, Q.~Kou, B.~Zhang, Y.~He, and B.~Rajwa, ``Simplicity of kmeans versus
  deepness of deep learning: A case of unsupervised feature learning with
  limited data,'' in \emph{IEEE ICMLA}, 2015, pp. 883--888.

\bibitem{CPY14ICASSP_short}
P.-Y. Chen and A.~O. Hero, ``Local {Fiedler} vector centrality for detection of
  deep and overlapping communities in networks,'' in \emph{IEEE ICASSP}, 2014,
  pp. 1120--1124.

\bibitem{peng2016recurrent}
X.~Peng, R.~S. Feris, X.~Wang, and D.~N. Metaxas, ``A recurrent encoder-decoder
  network for sequential face alignment,'' in \emph{ECCV}, 2016, pp. 38--56.

\bibitem{CPY17_Laplacian_tradeoff}
P.-Y. Chen and S.~Liu, ``Bias-variance tradeoff of graph laplacian
  regularizer,'' \emph{IEEE Signal Processing Letters}, vol.~24, no.~8, pp.
  1118--1122, Aug 2017.

\bibitem{peng2015circle}
X.~Peng, J.~Huang, Q.~Hu, S.~Zhang, A.~Elgammal, and D.~Metaxas, ``From circle
  to 3-sphere: Head pose estimation by instance parameterization,''
  \emph{Computer Vision and Image Understanding}, vol. 136, pp. 92--102, 2015.

\bibitem{liu2017accelerated}
S.~Liu, P.-Y. Chen, and A.~O. Hero, ``Accelerated distributed dual averaging
  over evolving networks of growing connectivity,'' \emph{arXiv preprint
  arXiv:1704.05193}, 2017.

\bibitem{qin2013regularized}
T.~Qin and K.~Rohe, ``Regularized spectral clustering under the
  degree-corrected stochastic blockmodel,'' in \emph{Advances in neural
  information processing systems (NIPS)}, 2013, pp. 3120--3128.

\bibitem{gao2016community}
C.~Gao, Z.~Ma, A.~Y. Zhang, and H.~H. Zhou, ``Community detection in
  degree-corrected block models,'' \emph{arXiv preprint arXiv:1607.06993},
  2016.

\bibitem{CPY14spectral}
P.-Y. Chen and A.~O. Hero, ``Phase transitions in spectral community
  detection,'' \emph{{IEEE} Trans. Signal Process.}, vol.~63, no.~16, pp.
  4339--4347, Aug 2015.

\bibitem{CPY17MIMOSA}
------, ``Multilayer spectral graph clustering via convex layer aggregation:
  Theory and algorithms,'' \emph{{IEEE} Trans. Signal Inf. Process. Netw.},
  2017.

\bibitem{HornMatrixAnalysis}
R.~A. Horn and C.~R. Johnson, \emph{{Matrix Analysis}}.\hskip 1em plus 0.5em
  minus 0.4em\relax Cambridge University Press, 1990.

\bibitem{Resnick13}
S.~Resnick, \emph{A Probability Path}.\hskip 1em plus 0.5em minus 0.4em\relax
  Birkh{\"a}user Boston, 2013.

\bibitem{Boyd04}
S.~Boyd and L.~Vandenberghe, \emph{Convex Optimization}.\hskip 1em plus 0.5em
  minus 0.4em\relax Cambridge University Press, 2004.

\end{thebibliography}

\setcounter{equation}{0}
\setcounter{figure}{0}
\setcounter{table}{0}
\setcounter{page}{1}
\makeatletter
\renewcommand{\theequation}{S\arabic{equation}}
\renewcommand{\thefigure}{S\arabic{figure}}
\section*{{\LARGE Supplementary Material}}
\appendices
\section{Proof of Lemma 3.1 }
\label{proof_SBM_cocentration}
We separate the proof into two cases: (I) $i \neq j$, and (II) $i = j$. For case (I), notice that under SBM($K,\bP$) each entry in $\bA_{ij}$ is an independent and identical Bernoulli random variable with success probability $P_{ij}$. Let 
$\bDelta=\bA_{ij}-\bAbar_{ij}$, where $\bAbar_{ij}= P_{ij} \bone_{n_i} \bone_{n_j}^T$.
As a result, each entry in $\bDelta$ is either $1-P_{ij}$ with probability $P_{ij}$ or $-P_{ij}$ with probability $1-P_{ij}$.
The Latala's theorem \cite{Latala05} states that for any random matrix $\mathbf{M}$ with statistically independent and zero mean entries, there exists a positive constant $c_1$ such that
\begin{align}
\mathbb{E} \Lb \sigma_1(\mathbf{M})\Rb &\leq c_1 \lb \max_s \sqrt{\sum_\ell \mathbb{E} \Lb [\mathbf{M}]_{s \ell}^2 \Rb}+\max_\ell \sqrt{\sum_s \mathbb{E} \Lb [\mathbf{M}]_{s \ell}^2 \Rb} \right. \nonumber \\
&~~~\left.+ \sqrt[4]{\sum_{s \ell} \mathbb{E} \Lb [\mathbf{M}]_{s \ell}^4 \Rb}\rb,
\end{align}
where $\sigma_1(\bM)$ is the largest singular value of $\bM$.
It is clear that each entry in $\bDelta$ is independent and has zero mean. By replacing  $\mathbf{M}$ with $\frac{\bDelta}{\sqrt{n_i n_j}}$ in the Latala's theorem, since $P_{ij} \in [0,1]$, we have $\max_s\sqrt{\sum_\ell \mathbb{E} \Lb [\mathbf{M}]_{s \ell}^2 \Rb}=O(\frac{1}{\sqrt{n_i}})$, $\max_\ell \sqrt{\sum_s \mathbb{E} \Lb [\mathbf{M}]_{s \ell}^2 \Rb}=O(\frac{1}{\sqrt{n_j}})$, and $\sqrt[4]{\sum_{s \ell} \mathbb{E} \Lb [\mathbf{M}]_{s \ell}^4 \Rb}=O(\frac{1}{\sqrt[4]{n_i n_j}})$.
Therefore, $\mathbb{E} \Lb \sigma_1\lb \frac{\bDelta}{\sqrt{n_i n_j}} \rb \Rb \ra 0$ as $n_i,n_j \ra \infty$.

We then use the Talagrand's concentration inequality, which is stated as follows. Let $h: \mathbb{R}^k \mapsto \mathbb{R}$ be a convex and 1-Lipschitz function.
Let $\bx \in \mathbb{R}^k$ be a random vector and assume that every element of $\bx$ satisfies
$|[\bx]_i|  \leq C$ for all $i=1,2,\ldots,k$, with probability one.
Then there exist positive constants $c_2$ and $c_3$ such that $\forall \epsilon >0$,
\begin{align}
\text{Pr}\lb \left| h(\bx)-\mathbb{E} \Lb h(\bx)  \Rb \right| \geq \epsilon\rb \leq c_2 \exp \lb \frac{-c_3 \epsilon^2}{C^2} \rb.
\end{align}
Since $\sigma_1(\mathbf{M})=\max_{\bz^T\bz=1}||\mathbf{M} \bz||_2$ \cite{HornMatrixAnalysis}, it is easy to check that  $\sigma_1(\mathbf{M})$ is a convex and 1-Lipschitz function. Therefore, applying the Talagrand's inequality and substituting $\mathbf{M}=\frac{\bDelta}{\sqrt{n_i n_j}}$ with the facts that $\mathbb{E} \Lb \sigma_1\lb \frac{\bDelta}{\sqrt{n_i n_j}} \rb \Rb \ra 0$ and $\frac{[\bDelta]_{s \ell}}{\sqrt{n_i n_j}} \leq \frac{1}{\sqrt{n_i n_j}}$, we have
\begin{align}
\text{Pr}\lb  \sigma_1 \lb \frac{\bDelta}{\sqrt{n_i n_j}} \rb \geq \epsilon \rb \leq c_2 \exp \lb -c_3 n_i n_j \epsilon^2 \rb.
\end{align}
Note that, since $n_i n_j \geq \frac{n_i+n_j}{2}$
for any positive integer $n_i,n_j >0$,   we have
$\sum_{n_i,n_j} c_2 \exp \lb -c_3 n_i n_j \epsilon^2 \rb < \infty$.
Hence, by the Borel-Cantelli lemma \cite{Resnick13},
$\sigma_1 \lb \frac{\bDelta}{\sqrt{n_i n_j}} \rb \asconv 0$
when $\none, \ntwo \ra \infty$, where $\asconv$ denotes almost sure convergence.
Using the result from standard matrix perturbation theory \cite{HornMatrixAnalysis} yields
$|\sigma_k(\bAbar_{ij}+\bDelta)-\sigma_k(\bAbar_{ij})| \leq \sigma_1(\bDelta)$
for all $k$, where $\sigma_k$ denotes the $k$-th largest singular value. By the fact that  $\sigma_1\lb \frac{\bDelta}{\sqrt{ n_i n_j}} \rb \asconv 0$,  we have as $n_i,n_j \ra \infty$,
\begin{align}
&\sigma_1\lb \frac{\bA_{ij}}{\sqrt{n_i n_j}} \rb=\sigma_1\lb \frac{\bAbar_{ij}+\bDelta}{\sqrt{n_i n_j}} \rb \asconv \sigma_1\lb \frac{\bAbar_{ij}}{\sqrt{n_i n_j}} \rb=P_{ij}; \\
&\sigma_k \lb \frac{\bA_{ij}}{\sqrt{n_i n_j}} \rb \asconv 0,~\forall~k \geq 2,
\end{align}
which implies 
\begin{align}
\frac{\bA_{ij}}{\sqrt{n_i n_j}} \asconv P_{ij} \frac{\bone_{n_i}}{\sqrt{n_i}} \frac{\bone_{n_j}^T}{\sqrt{n_j}} 
\end{align}
as $n_i,n_j \ra \infty$. Finally, for every $i,j \in \{1,\ldots,K\}$, $i \neq j$,
\begin{align}
\frac{\bA_{ij}}{n} = \frac{\bA_{ij}}{\sqrt{n_i n_j}} \cdot \frac{\sqrt{n_i n_j}}{n}  \asconv P_{ij} \sqrt{\rho_i \rho_j}  \frac{\bone_{n_i}}{\sqrt{n_i}} \frac{\bone_{n_j}^T}{\sqrt{n_j}} 
\end{align}
as $n_i,n_j \ra \infty$ and $\frac{\nmin}{\nmax} \ra c>0$, which completes the proof of case (I).

For case (II), since $\bA_{ii}$ accounts for the adjacency matrix of edges within community $i$, it is a symmetric matrix with zeros on its main diagonal. Let $\bAt$ denote the matrix that has the same entries as $\bA_{ii}$ in the upper diagonals and has zero entries in the lower diagonals. Then $\bA_{ii}=\bAt+\bAt^T$. 
Applying the Latala's theorem and the Talagrand's concentration inequality to $\frac{\bAt}{n_i}$, we have
\begin{align}
\frac{\bA_{ii}}{n_i} = \frac{\bAt+\bAt^T+ P_{ii} \bI_{n_k}}{n_i}  - \frac{P_{ii} \bI_{n_k}}{n_i} \asconv P_{ii} \frac{\bone_{n_i}}{\sqrt{n_i}} \frac{\bone_{n_i}^T}{\sqrt{n_i}} 
\end{align}
as $n_i \ra \infty$ due to the fact that $\frac{P_{ii} \bI_{n_k}}{n_i} \ra \bO$, where $\bO$ denotes the matrix of zeros. As a result, for every $i \in \{1,\ldots,K\}$,
\begin{align}
\frac{\bA_{ii}}{n} = \frac{\bA_{ii}}{n_i} \cdot \frac{n_i}{n}  \asconv \rho_i P_{ii} \frac{\bone_{n_i}}{\sqrt{n_i}} \frac{\bone_{n_i}^T}{\sqrt{n_i}} 
\end{align}
as $n_i \ra \infty$ and $\frac{\nmin}{\nmax} \ra c>0$, which completes the proof of case (II).

\section{Proof of Theorem 3.2}
\label{proof_SGC_condition}
\subsection{Optimality condition of $\bY$: general case}
\label{proof_SGC_condition_general}
Recall from Sec. \ref{subsec_NGL} that the eigenvector matrix $\bY=[\by_2~\ldots~\by_K]$ of $\bLN$ is the solution of the minimization problem 
\begin{align}
\label{eqn_SGC_3}
\min_{\bX \in \bbR^{n \times (K-1)},~\bX^T \bX = \bI_{K-1},~\bX^T \bD^{\frac{1}{2}} \bone_n=\bzero_{K-1} } \trace(\bX^T \bLN \bX).
\end{align}
Using (\ref{eqn_SGC_3}), we can write the Lagrangian function $\Gamma(\bX)$ of the minimization problem as 
\begin{align}
\label{eqn_Lagrangian_multi}
\Gamma(\bX)&=\trace(\bX^T \bLN \bX)-\bnu^T \bX^T \bD^{\frac{1}{2}}\bone_n \nonumber \\
&~~~- \trace \lb \bU (\bX^T \bX-\bI_{K-1}) \rb,
\end{align}
where
$\bnu \in \mathbb{R}^{K-1}$ and $\bU \in \mathbb{R}^{(K-1) \times (K-1)}$ with $\bU=\bU^T$ are the Lagrange multiplier of the constraints $\bX^T \bD^{\frac{1}{2}} \bone_n=\bzero_{K-1}$ and $\bX^T \bX= \bI_{K-1}$, respectively.

Using the Karush-Kuhn-Tucker (KKT) conditions \cite{Boyd04}, $\bY$ satisfies the first-order optimality condition of $\Gamma(\bX)$. That is, using matrix calculus and
differentiating (\ref{eqn_Lagrangian_multi}) with respect to $\bX$, we have 
\begin{align}
\label{eqn_opt_1}
\frac{d\Gamma(\bX)}{d \bX}= 2 \bLN \bX - \bD^{\frac{1}{2}} \bone_n \bnu^T - 2 \bX \bU,
\end{align}
which implies the optimality condition of $\bY$ is
\begin{align}
\label{eqn_opt_2}
2 \bLN \bY - \bD^{\frac{1}{2}} \bone_n \bnu^T - 2 \bY \bU = \bO_{n \times (K-1)},
\end{align}
where $\bO_{n \times K}$ denotes the $n \times K$ matrix of zeros.
Furthermore, left multiplying (\ref{eqn_opt_2}) by $(\bD^{\frac{1}{2}} \bone_n)^T$, we obtain 
\begin{align}
\label{eqn_opt_3}
\bone_n^T \bD \bone_n \bnu^T = \bzero_{K-1}^T
\end{align}
due to that fact that $\bLN \bD^{\frac{1}{2}} \bone_n = \bzero_{n}$ and $\bY^T \bD^{\frac{1}{2}} \bone_n=\bzero_{K-1}$.
Since $\bone_n^T \bD \bone_n=2m>0$ is the total degree of the graph, from (\ref{eqn_opt_3}) we conclude that $\bnu = \bzero_{K-1}$, which in turn simplifies the optimality condition in (\ref{eqn_opt_2}) as
\begin{align}
\label{eqn_opt_4}
\bLN \bY = \bY \bU.
\end{align}
Let $\Lambda=\diag([\lambda_2,\ldots,\lambda_K])$ be a diagonal matrix of the eigenvalues $\{\lambda_k\}_{k=2}^K$. Left multiplying (\ref{eqn_opt_4}) by $\bY^T$, we obtain 
\begin{align}
\label{eqn_opt_5}
\bU=\bY^T \bLN \bY = \Lambda
\end{align}
due to the fact that $\bY^T \bY = \bI_{K-1}$.

To investigate the relationship between $\bY$ and the community structure, we denote the rows in $\bY$ indexed by the nodes in community $k$ by
$\bY_k \in \bbR^{n_k \times (K-1)}$ such that $\bY=[\bY_1^T~\ldots~\bY_K^T]^T$. We use similar notation for $\bX$ such that $\bX=[\bX_1^T~\ldots~\bX_K^T]^T$. Furthermore, the matrix $\bLN$ is partitioned into a $K \times K$ block matrix, where the block ${\bLN}_{ij}$ is the $n_i \times n_j$ submatrix of $\bLN$ indexed by the community labels $i$ and $j$, $i,j \in \{1,\ldots,K\}$.
Given the fact that $\bnu=\bzero_{K-1}$, the Lagrangian function in (\ref{eqn_Lagrangian_multi}) can be written in terms of $\{ \bX_k\}_{k=1}^K$ and $\{ {\bLN}_{ij}\}_{i,j=1}^K$, which is
\begin{align}
\label{eqn_Lagrangian_multi_2}
\Gamma(\bX)&=\sum_{i=1}^K \sum_{j=1}^K \trace(\bX_i^T {\bLN}_{ij} \bX_j) - \sum_{i=1}^K \trace \lb \bU \bX_i^T \bX_i \rb \nonumber\\
&~~~+ \trace(\bU).
\end{align}
Differentiating $\Gamma(\bX)$ with respect to $\bX_k$, we have 
\begin{align}
\label{eqn_opt_6}
\frac{d\Gamma(\bX)}{d \bX_k}= 2 {\bLN}_{kk} \bX_k + 2 \sum_{j=1,j \neq k}^K {\bLN}_{kj} \bX_j - 2 \bX_k \bU,
\end{align}
which implies the optimality condition of $\bY$ in terms of $\{\bY_k\}_{k=1}^K$ is
\begin{align}
\label{eqn_opt_7}
&{\bLN}_{kk} \bY_k + \sum_{j=1,j \neq k}^K {\bLN}_{kj} \bY_j -  \bY_k \bU=\bO_{n_k \times (K-1)}, \nonumber \\
&~~~\forall~k \in \{1,\ldots,K\}.
\end{align}

\subsection{Optimality condition of $\bY$ under SBM($K$,$\bP$)}
Here we proceed to study the optimality condition of $\bY$ developed in Appendix \ref{proof_SGC_condition_general} under the assumption of SBM($K$,$\bP$). Unless specified, all convergence results are with respect to the condition when $n_k \ra \infty$, $\forall~k \in \{1,\ldots,K\}$, and $\frac{\nmin}{\nmax} \ra c>0$. 

Using the block representation $\{\bA_{ij}\}_{i,j=1}^K$ for the adjacency matrix $\bA$, define the block representation of the diagonal degree matrix $\bD$ as $\bD=[\bD_1^T~\ldots~\bD_K^T]^T$, where $\bD_k=\diag(\sum_{j=1}^{K} \bA_{kj} \bone_{n_j})$. Using Lemma \ref{lemma_SBM_concentration}, we have for every $k \in \{1,\ldots,K\}$,
\begin{align}
\label{eqn_SBM_degree}
\frac{\bD_k}{n}=\frac{\diag(\sum_{j=1}^{K} \bA_{kj} \bone_{n_j})}{n} \asconv \sum_{j=1}^K \sqrt{\rho_k \rho_j} P_{ij} \bI_{n_k}.
\end{align}
Similarly, the block representation of the unnormalized graph Laplacian matrix $\bL=\bD-\bA$, denoted by $\{\bL_{ij}\}_{i,j=1}^K$, has the following relation based on Lemma \ref{lemma_SBM_concentration}.
\begin{align}
\label{eqn_SBM_UGL}
\frac{\bL_{ij}}{n}=
\left\{
\begin{array}{ll}
\frac{\bD_i}{n}-\frac{\bA_{ii}}{n} \asconv  \sum_{k=1}^K \sqrt{\rho_i \rho_k} P_{ik} \bI_{n_i} - \rho_i P_{ii} \frac{\bone_{n_i}}{\sqrt{n_i}} \frac{\bone_{n_i}^T}{\sqrt{n_i}}, \\ \text{~if~} i=j; \\
-\frac{\bA_{ij}}{n} \asconv - \sqrt{\rho_i \rho_j} P_{ij} \frac{\bone_{n_i}}{\sqrt{n_i}} \frac{\bone_{n_j}^T}{\sqrt{n_j}}, \text{~if~} i \neq j.
\end{array}
\right.
\end{align}
Putting these pieces together, since by definition
\begin{align}
\bLN=\bD^{-\frac{1}{2}} \bL \bD^{-\frac{1}{2}}=\lb \frac{\bD}{n} \rb ^{-\frac{1}{2}} \frac{\bL}{n} \lb \frac{\bD}{n}\rb^{-\frac{1}{2}},
\end{align}
the blocks $\{ {\bLN}_{ij}\}_{i,j=1}^K$ of $\bLN$ satisfy
\begin{align}
\label{eqn_SBM_NGL}
\frac{{\bLN}_{ij}}{n}&=\lb \frac{\bD_i}{n} \rb ^{-\frac{1}{2}} \frac{\bL_{ij}}{n} \lb \frac{\bD_j}{n}\rb^{-\frac{1}{2}} \\
& \asconv
\left\{
\begin{array}{ll}
\bI_{n_i} - \frac{\rho_i P_{ii} \frac{\bone_{n_i}}{\sqrt{n_i}} \frac{\bone_{n_i}^T}{\sqrt{n_i}}}{a_i^2}, & \text{~if~} i=j; \\
\frac{- \sqrt{\rho_i \rho_j} P_{ij} \frac{\bone_{n_i}}{\sqrt{n_i}} \frac{\bone_{n_j}^T}{\sqrt{n_j}}}{a_i a_j}, &\text{~if~} i \neq j,
\end{array}
\right.
\end{align}
where  $a_k>0$ is defined as 
\begin{align}
\label{eqn_SBM_NGL_2}
a_k=\sqrt{ \sum_{j=1}^K \sqrt{\rho_k \rho_j} P_{kj}},~\forall~k \in \{1,\ldots,K\}.
\end{align}

Applying (\ref{eqn_SBM_NGL}) to the optimality condition in (\ref{eqn_opt_7}) gives
\begin{align}
\label{eqn_opt_SBM_1}
&\bY_k - \sum_{j=1}^K \frac{\sqrt{\rho_k \rho_j} P_{kj} \frac{\bone_{n_k}}{\sqrt{n_k}} \frac{\bone_{n_j}^T}{\sqrt{n_j}} \bY_j}{a_k a_j} - \bY_k \bU \asconv \bO_{n_k \times (K-1)}, \nonumber \\
&~~~\forall~k \in \{1,\ldots,K\}.
\end{align}
 As a result, (\ref{eqn_opt_SBM_1}) is the asymptotic optimality condition of $\bY$ under SBM($K,\bP$).
In the sequel we will use  (\ref{eqn_opt_SBM_1}) to specify the feasibility of community detection using $\bY$.

\subsection{Undetectable regime for community detection}
\label{subproof_infeasible}
Left multiplying the optimality condition (\ref{eqn_opt_SBM_1}) of $\bY$ under SBM($K,\bP$) by $\frac{\bone_{n_k}^T}{\sqrt{n_k}}$, we have 
\begin{align}
\label{eqn_opt_SBM_2}
&\frac{\bone_{n_k}^T \bY_k}{\sqrt{n_k}}  ( \bI_{K-1}- \bU) - \sum_{j=1}^K \frac{\sqrt{\rho_k \rho_j} P_{kj} \frac{\bone_{n_j}^T}{\sqrt{n_j}} \bY_j}{a_k a_j} \asconv \bzero_{K-1}^T, \nonumber \\
&~~~\forall~k \in \{1,\ldots,K\}.
\end{align}
A trivial solution for $\{\bY_k\}_{k=1}^K$ to satisfy (\ref{eqn_opt_SBM_2}) is
\begin{align}
\label{eqn_SBM_infeasible}
 \frac{\bY_k^T \bone_{n_k}}{\sqrt{n_k}}  \asconv \bzero_{K-1},~\forall~k \in \{1,\ldots,K\}.
\end{align}
Moreover, applying (\ref{eqn_SBM_infeasible}) to (\ref{eqn_opt_SBM_1}) gives
\begin{align}
\label{eqn_opt_SBM_3}
\bY_k (\bI_{K-1}-\bU) \asconv \bO_{n_k \times (K-1)},~\forall~k \in \{1,\ldots,K\}.
\end{align}
Since the orthogonality and unit-norm constraint $\bY^T \bY = \bI_{K-1}$ is equivalent to $\sum_{k=1}^{K} \bY_k^T \bY_k=\bI_{K-1}$,  
left multiplying (\ref{eqn_opt_SBM_3}) by $\bY_k^T$ and summing over $k=1,\ldots,K$ gives
\begin{align}
\label{eqn_opt_SBM_4}
\bU \asconv \bI_{K-1}
\end{align}
when (\ref{eqn_SBM_infeasible}) holds. Conversely, if (\ref{eqn_opt_SBM_4}) holds, applying it to  (\ref{eqn_opt_SBM_1}) gives
\begin{align}
\label{eqn_opt_SBM_5}
&\sum_{j=1}^K \frac{\sqrt{\rho_k \rho_j} P_{kj} \frac{\bone_{n_k}}{\sqrt{n_k}} \frac{\bone_{n_j}^T}{\sqrt{n_j}} \bY_j}{a_k a_j} \asconv \bO_{n_k \times (K-1)},\nonumber\\
&~~~\forall~k \in \{1,\ldots,K\}.
\end{align}
Since (\ref{eqn_opt_SBM_5}) holds for any positive $\{\rho_k\}_{k=1}^K$, $\{a_k\}_{k=1}^K$ and $\{P_{kj}\}_{k,j=1}^K$, it implies the condition in (\ref{eqn_SBM_infeasible}). Consequently, we have established that $\frac{\bY_k^T \bone_{n_k}}{\sqrt{n_k}}   \asconv \bzero_{K-1},~\forall~k \in \{1,\ldots,K\}$ if and only if $ \bU \asconv \bI_{K-1}$.

Note that the condition in (\ref{eqn_SBM_infeasible}) shows the rows of $\bY_k$ sum to a zero vector for each $k$, which implies the row representation of nodes in the same community is incoherent. That is, for each column in $\bY_k$, the sum of nonzero entries is zero, which implies the nonzero entries in each column have alternating signs and hence the row representation is not identical for nodes in the same community.
Furthermore, when one runs K-means clustering on the rows of $\bY$, the centroid of each community collapses to the same point due to   (\ref{eqn_SBM_infeasible}), which makes correct community detection impossible. Similar results can be concluded when one adopts the row normalization step and use $\bhY$ for community detection as described in Algorithm \ref{algo_SGC}, since the row normalization step does not alter the sign of each entry in $\bhY$. As a result, community detection using $\bLN$ is said to be in the undetectable regime if  (\ref{eqn_SBM_infeasible}) holds. 

Recall the definition $\theta=\sum_{k=2}^K 1-\lambda_k$ as defined in Theorem \ref{thm_NGL_SBM}.
Taking the trace on both sides in  (\ref{eqn_SBM_infeasible}) and using (\ref{eqn_opt_5}), we have 
\begin{align}
\trace(\bU)=\trace(\bLambda)=\sum_{k=2}^K \lambda_k=\trace(\bI_{K-1})=K-1,
\end{align}
which implies community detection using $\bY$ of $\bLN$ is undetectable if and only if $\theta=0$.

\subsection{Detectable regime for community detection}
Appendix \ref{subproof_infeasible} shows that the trivial solution (\ref{eqn_SBM_infeasible}) to the optimality condition under SBM($K,\bP$) in (\ref{eqn_opt_SBM_2}) results in incorrect community detection. Here we investigate the other solution to  (\ref{eqn_opt_SBM_1}) and show that this nontrivial solution leads to correct community detection, which is called the detectable regime for community detection.

Using (\ref{eqn_opt_SBM_1}), we can rewrite it as 
\begin{align}
\label{eqn_opt_SBM_6}
&\bY_k (\bI_{K-1} - \bU) \asconv \sum_{j=1}^K \frac{\sqrt{\rho_k \rho_j} P_{kj} \frac{\bone_{n_k}}{\sqrt{n_k}} \frac{\bone_{n_j}^T}{\sqrt{n_j}} \bY_j}{a_k a_j}, \nonumber \\
&~~~\forall~k \in \{1,\ldots,K\}.
\end{align}
Left multiplying $\bY_k^T$ to (\ref{eqn_opt_SBM_6}) and summing over $k \in \{1,\ldots,K\}$ gives
\begin{align}
\label{eqn_opt_SBM_6_2}
\bU \asconv \bI_{K-1} - \sum_{k=1}^K \sum_{j=1}^K \frac{\sqrt{\rho_k \rho_j} P_{kj} \bY_k^T \frac{\bone_{n_k}}{\sqrt{n_k}} \frac{\bone_{n_j}^T}{\sqrt{n_j}} \bY_j}{a_k a_j}.
\end{align}
Recall from (\ref{eqn_opt_5}) that $\bU=\bLambda$ and hence $\bU$ is a diagonal matrix with positive entries on its diagonal.
If the matrix $\bI_{K-1} - \bU$ is invertible, then by (\ref{eqn_opt_SBM_6}),
\begin{align}
\label{eqn_opt_SBM_7}
\bY_k &\asconv  \sum_{j=1}^K \frac{\sqrt{\rho_k \rho_j} P_{kj} \frac{\bone_{n_k}}{\sqrt{n_k}} \frac{\bone_{n_j}^T}{\sqrt{n_j}} \bY_j}{a_k a_j}  (\bI_{K-1} - \bU)^{-1} \\
\label{eqn_opt_SBM_7_2}
&= \frac{\bone_{n_k}}{\sqrt{n_k}} \bb_k^T,~\forall~k \in \{1,\ldots,K\},
\end{align}
where the $(K-1) \times 1$ vector $\bb_k$ is defined as
\begin{align}
\label{eqn_opt_SBM_8}
\bb_k=  (\bI_{K-1} - \bU)^{-1} \sum_{j=1}^K \frac{\sqrt{\rho_k \rho_j} P_{kj}   \bY_j^T \frac{\bone_{n_j}}{\sqrt{n_j}}}{a_k a_j}.
\end{align}
The result of (\ref{eqn_opt_SBM_7_2}) implies each block $\bY_k$ in $\bY$ has coherent row representation, which means the vector space representation of nodes in the same community is identical. The next step is to show that the row representation of each $\bY_k$ is distinct, and hence inspecting the distribution of rows in $\bY$ leads to correct community detection.

Using (\ref{eqn_opt_SBM_7_2}) and (\ref{eqn_SBM_degree}), the orthogonality and unit-norm constraints $\sum_{k=1}^K \bY_k^T \bY_k=\bI_{K-1}$ and $\bY^T \bD^{\frac{1}{2}} \bone_n=\bzero_{K-1}$ yield
\begin{align}
\label{eqn_opt_SBM_9}
\sum_{k=1}^K \bb_k \bb_k^T \asconv \bI_{K-1};\\
\label{eqn_opt_SBM_10}
\sum_{k=1}^K a_k \bb_k  \asconv \bzero_{K-1},
\end{align}
where $a_k>0$ is defined in (\ref{eqn_SBM_NGL_2}).

The result in (\ref{eqn_opt_SBM_9})
imply that some $\bb_k$ cannot be a zero vector since
\begin{align}
\label{eqn_opt_SBM_13}
\sum_{k=1}^K  [\bb_k]_j^2=1,~\forall~j \in\{1,\ldots,K-1\}.
\end{align}
Furthermore, by (\ref{eqn_opt_SBM_9}) and (\ref{eqn_opt_SBM_10}),
we have
\begin{align}
\label{eqn_opt_SBM_11}
\sum_{k:[\bb_k]_j>0} a_k [\bb_k]_j &= - \sum_{k: [\bb_k]_j <0}  a_k [\bb_k]_j,~\forall~j \in\{1,\ldots,K-1\}; \\
\label{eqn_opt_SBM_12}
\sum_{k:[\bb_k]_i [\bb_k]_j>0}  [\bb_k]_i [\bb_k]_j &= - \sum_{k: [\bb_k]_i [\bb_k]_j <0}  [\bb_k]_i [\bb_k]_j, \nonumber \\&~~~\forall~i,j \in\{1,2,\ldots,K-1\}, i \neq j.
\end{align}
Combining the results in (\ref{eqn_opt_SBM_7_2}), (\ref{eqn_opt_SBM_13}),  (\ref{eqn_opt_SBM_11}) and (\ref{eqn_opt_SBM_12}) leads to the following conclusion:
\begin{enumerate}
\item The columns of $\bY_k$ are constant vectors.
\item Each column of $\bY$ has at least two nonzero community-wise constant components, and these constants have alternating signs such that their weighted sum equals $0$ (i.e., $\sum_{k} a_k [\bb_k]_j = 0,~\forall~j \in\{1,\ldots,K-1\}$).
\item No two columns of $\bY$ have the same sign on the community-wise nonzero components.	
\end{enumerate}
As a result, we have proved that the rows in each $\bY_k$ have identical row representation, and the row representation of each $\bY_k$ is distinct. More importantly, the results suggest that in the vector space representation the within-cluster distance between any pair of row vectors in each $\bY_k$ is zero, whereas the between-cluster distance between any two row vectors of different clusters is nonzero. This suggests that the ground-truth communities are the optimal solution to K-means clustering, and
hence K-means clustering on the rows of $\bY$ can yield correct communities.

Since  $\bU=\bLambda$ by ($\ref{eqn_opt_5}$), comparing (\ref{eqn_opt_SBM_6_2}) in the detectable regime to (\ref{eqn_opt_SBM_4}) in the undetectable regime, one can see that the changes in the edge connection matrix $\bP$ lead to changes in the eigenvalue matrix $\Lambda$, and $\theta=\trace(I_{K-1}-\bU)=\sum_{k=2}^K 1 - \lambda_k \geq 0$.
As a result, the parameters $\bP$ can also be separated into the detectable and undetectable  regimes for community detection. 
Lastly, we have established correct community detection under SBM($K,\bP$) provided that $\bI_{K-1} - \bU$ is invertible, which implies $\theta>0$. Notice that $\bI_{K-1} - \bU$ is not invertible when $\bU \asconv \bI_{K-1} $, which leads to the undetectable  regime as discussed in Sec. \ref{subproof_infeasible}. Consequently, the $K$ communities  under SBM($K,\bP$) can be correctly detected using $\bY$ if and only if $\theta>0$.

\subsection{Summary: the distribution of the rows in $\bY$}
Here we summarize the established theoretical analysis for community detectability using the eigenvector matrix $\bY$ of the normalized graph Laplacian matrix $\bLN$ for graphs generated by $SBM(K,\bP)$.
\begin{enumerate}
\item The community-indexed block matrix $\bY_k$ of $\bY$ satisfies the optimality condition in (\ref{eqn_opt_SBM_1}). Moreover, the distribution of the rows in $\bY$ is either in the detectable regime or the undetectable regime for community detection.
\item In the undetectable regime, the rows in each $\bY_k$ sum to a zero vector, resulting in incorrect community detection. In addition, $\bY$ is in the undetectable community detection regime if and only if $\theta=\sum_{k=2}^K 1- \lambda_k=0$.
\item In the detectable regime, the rows in  each $\bY_k$ have identical representation, and each row representation of $\bY_k$ is distinct, resulting in correct community detection using K-means clustering on the rows of $\bY$.  In addition, $\bY$ is in the detectable community detection regime if and only if $\theta>0$.
\end{enumerate}

\section{Proof of Corollary 3.3}
\label{proof_SBM_two}
When restricted to the case of SBM(2,$\bP$), where $P_{11}=p_1$, $P_{22}=p_2$ and $P_{12}=P_{21}=q$, the parameter $a_k$ in (\ref{eqn_SBM_NGL_2}) can be simplified to 
\begin{align}
\label{eqn_SBM_two_1}
&a_1=\sqrt{ \rho_1 p_1 + \sqrt{\rho_1 \rho_2}q}; \\
\label{eqn_SBM_two_1_2}
&a_2=\sqrt{ \rho_2 p_2 + \sqrt{\rho_1 \rho_2}q}. 
\end{align}		
In addition, using (\ref{eqn_SBM_degree}) and the orthogonality constraint $\by^T \bD^{\frac{1}{2}} \bone_n=0$ of the second smallest eigenvector $\by_2=[\byt_1^T~\byt_2^T]^T$ gives
\begin{align}
\label{eqn_SBM_two_2}
a_1 \byt_1^T \frac{\bone_{n_1}}{\sqrt{n_1}} + a_2 \byt_2^T \frac{\bone_{n_2}}{\sqrt{n_2}} \asconv 0.
\end{align}	
Applying (\ref{eqn_SBM_two_2}) to (\ref{eqn_opt_SBM_6_2}) gives
\begin{align}
\label{eqn_SBM_two_3}
\bU &\asconv 1-\frac{\rho_1 p_1 (\byt_1^T  \frac{\bone_{n_1}}{\sqrt{n_1}})^2 }{a_1^2}-\frac{\rho_2 p_2 a_1^2 (\byt_1^T  \frac{\bone_{n_2}}{\sqrt{n_2}})^2 }{a_2^4} \nonumber \\
&~~~+ \frac{2 \sqrt{\rho_1 \rho_2} q (\byt_1^T  \frac{\bone_{n_1}}{\sqrt{n_1}})^2}{a_2^2} \\
\label{eqn_SBM_two_3_2}
&= 1 - \frac{\Lb \lb \frac{a_2}{a_1} \rb^2 \rho_1 p_1 + \lb \frac{a_1}{a_2} \rb^2 \rho_2 p_2  - 2  \sqrt{\rho_1 \rho_2} q \Rb(\byt_1^T  \frac{\bone_{n_1}}{\sqrt{n_1}})^2}{a_2^2}.
\end{align}	
Substituting (\ref{eqn_SBM_two_1}) and (\ref{eqn_SBM_two_1_2}) to (\ref{eqn_SBM_two_3_2}) and factoring the resulting term, we have
\begin{align}
\label{eqn_SBM_two_4}
\bU \asconv 
1 - \frac{\rho_1 \rho_2 \Lb 2 \sqrt{\rho_1 \rho_2} + \rho_1 p_1 + \rho_2 p_2 \Rb \lb p_1 p_2 - q^2 \rb (\byt_1^T  \frac{\bone_{n_1}}{\sqrt{n_1}})^2}{a_2^2}.
\end{align}	
Comparing (\ref{eqn_SBM_two_4}) to (\ref{eqn_opt_SBM_4}), SBM($2,\bP$) is in the detectable regime for community detection if and only if  $q<\sqrt{p_1 p_2}$, and it is in the undetectable regime for community detection if and only if $q \geq \sqrt{p_1 p_2}$.

Lastly, if $q<\sqrt{p_1 p_2}$, then from (\ref{eqn_opt_SBM_8}) we have 
\begin{align}
\label{eqn_SBM_two_5}
\byt_1 \asconv b_1  \frac{\bone_{n_1}}{\sqrt{n_1}};~\byt_2 \asconv b_2  \frac{\bone_{n_2}}{\sqrt{n_2}}
\end{align}
for some $b_1$ and $b_2$. Applying (\ref{eqn_SBM_two_5}) to (\ref{eqn_SBM_two_2}) and the unit-norm constraint $\by_2^T \by_2=\byt_1^T \byt_1 + \byt_2^T \byt_2=1$ gives
\begin{align}
\label{eqn_SBM_two_6}
b_1 = \frac{\pm a_2 }{\sqrt{a_1^2+a_2^2}};~b_2 = \frac{\mp a_1 }{\sqrt{a_1^2+a_2^2}}.
\end{align}
Setting $\beta_1=\frac{ a_2 }{\sqrt{a_1^2+a_2^2}}$ and $\beta_2=\frac{ a_1 }{\sqrt{a_1^2+a_2^2}}$ completes the proof.

\end{document}